\documentclass{article}

\usepackage[final,main]{neurips_2026}

\usepackage[utf8]{inputenc}
\usepackage[T1]{fontenc}
\usepackage{hyperref}
\usepackage{url}
\usepackage{mdwlist}
\usepackage{booktabs}
\usepackage{amsfonts}
\usepackage{amsmath}
\usepackage{amssymb}
\usepackage{amsthm}
\usepackage{mathtools}
\usepackage{nicefrac}
\usepackage{microtype}
\usepackage{xcolor}
\usepackage{enumitem}
\usepackage{graphicx}
\usepackage{subcaption}
\usepackage[capitalize,noabbrev]{cleveref}
\graphicspath{{figures/}}

\theoremstyle{plain}
\newtheorem{theorem}{Theorem}[section]
\usepackage{aliascnt}
\newaliascnt{proposition}{theorem}
\newtheorem{proposition}[proposition]{Proposition}
\aliascntresetthe{proposition}
\newaliascnt{corollary}{theorem}
\newtheorem{corollary}[corollary]{Corollary}
\aliascntresetthe{corollary}
\newaliascnt{lemma}{theorem}
\newtheorem{lemma}[lemma]{Lemma}
\aliascntresetthe{lemma}
\crefname{proposition}{Proposition}{Propositions}
\crefname{corollary}{Corollary}{Corollaries}
\crefname{lemma}{Lemma}{Lemmas}

\theoremstyle{definition}

\theoremstyle{remark}

\newcommand{\Hs}{\mathcal{H}}
\newcommand{\Xs}{\mathcal{X}}

\newcommand{\Zs}{\mathcal{Z}}
\newcommand{\R}{\mathbb{R}}
\newcommand{\E}{\mathbb{E}}
\newcommand{\Prb}{\mathrm{Pr}}
\newcommand{\Ind}{\mathbf{1}}
\newcommand{\inner}[2]{\langle #1,\, #2 \rangle}
\newcommand{\normH}[1]{\|#1\|_{\Hs}}

\DeclareMathOperator{\MMD}{MMD}

\DeclareMathOperator{\TPR}{TPR}
\DeclareMathOperator{\FPR}{FPR}
\DeclareMathOperator{\PPV}{PPV}
\DeclareMathOperator{\FNR}{FNR}
\DeclareMathOperator{\spanop}{span}

\title{The Pok\'emon Theorem: \\
  Fairness Criteria Can't Catch 'Em All}

\author{%
  Daniel Matsui Smola \\
  Department of Computer Science \\
  University of Washington \\
  Seattle, WA \\
  \texttt{daniel@matsuismola.com} \\
  \And
  Alex Smola \\
  Boson AI \\
  Santa Clara, CA \\
  \texttt{smola@boson.ai} \\
}

\begin{document}

\maketitle

\begin{abstract}
Debates around algorithmic fairness rarely end. Models considered ``fair'' by one criterion are soon shown to be unfair by another. We prove that this is unavoidable. Satisfying any finite checklist of fairness criteria never certifies fairness itself, since the model may always violate yet another criterion. We call this the Pok\'emon theorem: you can't catch them all. In fact, trying to catch all criteria recovers the classical impossibility theorems. Since these violations come from differences between the groups' data, a natural remedy would be to make the data fair by learning fair features. Unfortunately, this, too, is impossible: features that scrub out sensitive group information also scrub out all predictive signal. While exact fairness is out of reach, approximate fairness is not. We derive the Pareto frontier between fairness and predictive signal for approximately fair features, and show how to attain it.
\end{abstract}

\section{Introduction}
\label{sec:intro}

Fairness of a model is usually assessed with a checklist of fairness criteria that compare how it treats two groups (e.g., men and women). A practitioner might require that the model makes positive decisions at equal rates for both groups, makes errors at equal rates for both groups, and gives scores that are calibrated within each group. Yet it has long been known that such criteria can be incompatible. In the 1960s and 1970s, researchers proposed several definitions of fairness for educational and employment tests, and showed that they cannot all hold at the same time \citep{darlington1971another,thorndike1971concepts,cole1973bias,petersen1976evaluation}. The debate remains unresolved sixty years later \citep{hutchinson2019fifty}.

The debate intensified in 2016 with COMPAS, an algorithm used by US courts that gives each defendant a risk score predicting how likely they are to reoffend. ProPublica found that Black defendants who did not reoffend were labeled high risk nearly twice as often as white defendants \citep{angwin2016machine}. Northpointe, the developer of COMPAS, replied that the score was calibrated: among defendants with the same score, the same fraction reoffended in both groups \citep{dieterich2016compas}. Kleinberg, Mullainathan and Raghavan \citep[KMR;][]{kleinberg2017inherent} showed that these two fairness criteria are incompatible: when the groups reoffend at different rates (unequal base rates), no score can be calibrated within each group and also have equal error rates for both groups, unless it predicts perfectly. This result, together with that of \citet{chouldechova2017fair}, started the modern study of impossibility results in fairness.

Since then, many more fairness criteria have been proposed, along with many more impossibility theorems \citep{pleiss2017fairness,zhaogordon2022inherent,barocas2023fairness,stewart2024impossibility}. Each theorem shows that a particular combination of criteria cannot hold together. \emph{Is the problem specific to these combinations, or does it affect any finite checklist of criteria?}

We prove that satisfying any finite checklist of criteria cannot certify that an algorithm is fair, which is why this sixty-year-old debate has remained unresolved. The proof uses mean embeddings in a reproducing kernel Hilbert space (RKHS) to express the difference between the two groups' data distributions and each fairness criterion as vectors. Each criterion measures a dot product with this difference vector; a value of zero means that the criterion detects no difference between the groups and is satisfied. All criteria on a checklist can therefore be satisfied while the groups differ in directions orthogonal to the subspace the criteria span. If the two groups' data distributions were identical, the difference vector would be zero, and every criterion would be satisfied whatever the algorithm, as long as it does not use the group. Unfairness thus originates in the data, not in the algorithm. This raises a second question: \emph{can we learn features that scrub group information, yet retain predictive signal?}

We answer both questions through the following contributions:
\begin{enumerate}[leftmargin=1.5em, itemsep=1pt, topsep=2pt]
  \item \textbf{Pok\'emon Theorem.} No finite checklist of fairness criteria can certify fairness. We characterize the worst-case residual unfairness after the best $m$ criteria are satisfied. These best criteria are the leading eigendirections of the data's covariance operator, pooled over both groups (\cref{sec:pokemon}).
  \item \textbf{Recovering Classical Impossibilities.} Escaping the Pok\'emon theorem requires a dense checklist of criteria. We show that requiring one recovers the classical impossibility theorems, including KMR. With only the best $m$ criteria, the same conclusions hold approximately, with an error that shrinks as $m$ grows (\cref{sec:recover}).
  \item \textbf{Impossibility of Fair Feature Learning.} Fair representation learning seeks features that scrub group information, so that every model built on them is fair, while retaining predictive signal. We show that these two goals are incompatible when base rates differ (\cref{sec:fair-feature}).
  \item \textbf{Approximate Fairness.} While exact fairness is impossible, approximate fairness is not: features that scrub group information only partially can retain predictive signal. We derive the Pareto frontier between fairness and predictive signal. The more the base rates differ, the more fairness must be given up for the same signal. We also give a rule that attains this frontier (\cref{sec:approx}).
\end{enumerate}

\section{Background}
\label{sec:background}

\subsection{Fairness Criteria as Expectations}
\label{sec:criteria-expectations}

Each individual has features $X \in \Xs$, an outcome $Y \in \{0, 1\}$, such as whether a defendant reoffends, and a group $G \in \{a, b\}$. Write $P_g$ for the distribution of $X$ in group $g$. The \emph{base rate} $p_g := \Prb[Y = 1 \mid G = g]$ is the fraction of positive outcomes in group $g$, and $\Delta p := p_a - p_b$ is the \emph{base-rate gap}. We assume $\Prb[G = g] > 0$ and $0 < p_g < 1$ for both groups.

Many fairness criteria compare the expectation of a test function $c : \Xs \to \R$ in the two groups. A criterion is satisfied when the two expectations are equal,
\begin{equation}
  \E[c(X) \mid G = a] - \E[c(X) \mid G = b] \;=\; 0.
  \label{eq:criterion-expectation}
\end{equation}
For example, demographic parity requires a classifier $\hat Y = f(X)$ to select both groups at equal rates; its test function is $c(x) = \Ind[f(x) = 1]$ \citep{dwork2012fairness}. Other criteria compare the groups within each outcome class or involve the outcome itself; we introduce them in \cref{sec:pokemon-other}. For fixed $c$, the map $P \mapsto \E_P[c(X)]$ is linear in the distribution $P$, even when $c$ is nonlinear in $x$. We will use kernels to exploit this linearity. Nonlinear criteria are handled in \cref{sec:pokemon-other}, and learned features $Z = \Phi(X)$ in place of $X$ are handled in \cref{sec:fair-feature}.

\subsection{Kernels and Mean Embeddings}
\label{sec:kernels-embeddings}

Fix a positive-definite kernel $k : \Xs \times \Xs \to \R$ with RKHS $\Hs$ and feature map $\phi : \Xs \to \Hs$. The reproducing property gives $f(x) = \inner{f}{\phi(x)}_{\Hs}$ for $f \in \Hs$, and $k(x, x') = \inner{\phi(x)}{\phi(x')}_{\Hs}$ \citep{aronszajn1950theory,scholkopf2002learning}. On a compact $\Xs$, $k$ is \emph{universal} if $\Hs$ is dense in $C(\Xs)$, so every continuous function can be uniformly approximated by RKHS functions \citep{steinwart2001influence}.

For a distribution $P$ on $\Xs$ with $\E_P[\sqrt{k(X, X)}] < \infty$, the \emph{mean embedding} is $\mu_P := \E_P[\phi(X)] \in \Hs$. The reproducing property yields $\E_P[f(X)] = \inner{f}{\mu_P}_{\Hs}$ \citep{gretton2012kernel}. The kernel $k$ is \emph{characteristic} when $\mu_P = \mu_Q$ holds if and only if $P = Q$ \citep{sriperumbudur2010hilbert}; universality implies characteristic \citep{gretton2012kernel}, and Gaussian, Laplace and Mat\'ern kernels on $\R^d$ are characteristic. The \emph{maximum mean discrepancy} (MMD) between $P$ and $Q$ is
\begin{equation}
  \MMD(P, Q) \;=\; \sup_{\normH{f} \le 1}\bigl(\E_P[f(X)] - \E_Q[f(X)]\bigr) \;=\; \normH{\mu_P - \mu_Q},
  \label{eq:mmd}
\end{equation}
where the \emph{witness function} $f^\star = (\mu_P - \mu_Q) / \normH{\mu_P - \mu_Q}$ attains the supremum for $\mu_P \ne \mu_Q$ \citep{gretton2012kernel}. For a characteristic kernel, $\MMD(P, Q) = 0$ if and only if $P = Q$.

\subsection{Fairness Criteria as Directions}
\label{sec:criteria-directions}

Write the mean embedding of $P_g$ and the \emph{conditional mean discrepancy} between the groups as
\begin{equation}
  \mu_g \;:=\; \E[\phi(X) \mid G = g],
  \qquad
  \delta \;:=\; \mu_a - \mu_b.
  \label{eq:cmd}
\end{equation}
For a criterion whose test function lies in the RKHS, $c(x) = \inner{v}{\phi(x)}_{\Hs}$, the reproducing property gives
\begin{equation}
  \E[c(X) \mid G = a] - \E[c(X) \mid G = b] \;=\; \inner{v}{\delta}_{\Hs}.
  \label{eq:criterion-dot-product}
\end{equation}
Satisfying the criterion means $\inner{v}{\delta}_{\Hs} = 0$: its direction $v$ is orthogonal to $\delta$. By construction, $\MMD(P_a, P_b) = \normH{\delta}$ is the largest difference between the groups that any criterion with $\normH{v} \le 1$ can detect. For a characteristic kernel, $\delta = 0$ means that the groups' distributions of $X$ are equal, so every criterion that compares them is satisfied. Criteria that compare other distributions have their own discrepancies; those within each outcome class (\cref{sec:pokemon-other}) use the class-conditional versions
\begin{equation}
  \mu_{y, g} \;:=\; \E[\phi(X) \mid Y = y, G = g],
  \qquad
  \delta_y \;:=\; \mu_{y, a} - \mu_{y, b},
  \label{eq:cmd-class}
\end{equation}
where $\delta_y$ is the conditional mean discrepancy between the groups among individuals with outcome $y$.

\subsection{Related Work}
\label{sec:related-work}

\citet{kleinberg2017inherent}, \citet{chouldechova2017fair} and \citet{pleiss2017fairness} show that calibration and equal error rates are incompatible when base rates differ. \citet{stewart2024impossibility} show that equal error rates are also incompatible with a weaker form of calibration, and \citet{barocas2023fairness} show that every pair among demographic parity, equal error rates and calibration is incompatible. Each result fixes a combination of criteria. We consider every finite checklist, and recover these results as the checklist grows (\cref{sec:recover}). \citet{kearns2018preventing} show that auditing a few subgroups can miss unfairness on others; we vary the criteria instead of the subgroups.

Fair representation learning seeks features that scrub group information while retaining predictive signal \citep{zemel2013learning,edwards2016censoring,louizos2016variational,madras2018learning,deka2023mmdbfair}. \citet{zhaogordon2022inherent} show that, when base rates differ, any classifier built on such features must make errors, and \citet{lechner2021impossibility} show that no features can be fair for every downstream task. We show that, when base rates differ, scrubbing group information both overall and within each outcome class removes all predictive signal (\cref{sec:fair-feature}). \citet{menon2018cost} and \citet{wei2022pareto} study the frontier between fairness and accuracy for a single classifier. We solve the Pareto frontier between fairness and predictive signal for features, whose slope is set by the base-rate gap (\cref{sec:approx}).

Kernel methods have so far been used to enforce fairness \citep{perezsuay2017fair,perezsuay2023fairkernel,donini2018empirical,wei2023mean}. We use the same embeddings to show what finitely many constraints cannot certify.

\section{The Pok\'emon Theorem}
\label{sec:pokemon}

\subsection{Qualitative Theorem}
\label{sec:pokemon-qualitative}

We return to the question raised in the introduction: \emph{is the problem specific to particular combinations of criteria, or does it affect any finite checklist?} We prove that it affects any finite checklist, even when every criterion on it is satisfied exactly. We call this the Pok\'emon theorem: \emph{gotta catch 'em all}, but another criterion always escapes.

\begin{theorem}[Pok\'emon theorem]
\label{thm:pokemon}
Let $k$ be a characteristic kernel with RKHS $\Hs$, and suppose the groups differ, $P_a \ne P_b$. Let the checklist consist of $m$ criteria with test functions $c_i(x) = \inner{v_i}{\phi(x)}_{\Hs}$, all of which are satisfied:
\[
  \E[c_i(X) \mid G = a] - \E[c_i(X) \mid G = b] \;=\; \inner{v_i}{\delta}_{\Hs} \;=\; 0
  \qquad \text{for all } i = 1, \ldots, m,
\]
or equivalently $\delta \perp V$ with $V := \spanop\{v_1, \ldots, v_m\}$. Then there is a criterion $v_{m+1} \in \Hs$, orthogonal to $V$, that is violated: $\inner{v_{m+1}}{\delta}_{\Hs} > 0$. An explicit choice is the MMD witness $v_{m+1} := \delta / \normH{\delta}$, for which
\begin{equation}
  \inner{v_{m+1}}{\delta}_{\Hs} \;=\; \normH{\delta} \;=\; \MMD(P_a, P_b).
  \label{eq:witness}
\end{equation}
\end{theorem}

\begin{proof}
Satisfying every criterion places $\delta$ in $V^\perp$. Since the kernel is characteristic and $P_a \ne P_b$, $\delta \ne 0$. Hence $\delta / \normH{\delta}$ is a unit vector in $V^\perp$ with $\inner{\delta / \normH{\delta}}{\delta}_{\Hs} = \normH{\delta} > 0$.
\end{proof}

\paragraph{Implications.}
Satisfying every criterion on a checklist does not make an algorithm fair. For every finite checklist, groups that satisfy it yet differ exist (\cref{lem:realizability}). Whenever the groups' data distributions differ, some other criterion is violated, and satisfying it as well only leads to another, since the theorem applies to every finite checklist. Unlike the classical impossibilities, this holds even when the groups have equal base rates.

\subsection{Quantitative Theorem}
\label{sec:pokemon-quantitative}

Satisfying more criteria should nonetheless buy us something. \emph{How large can the remaining unfairness be if we satisfy the best $m$ criteria?}

When every criterion is satisfied, $\delta \perp V$, and the largest violation of any criterion with $\normH{v} \le 1$ is $\normH{\delta} = \MMD(P_a, P_b)$, attained by the witness. We call it the \emph{residual unfairness}.

Although $\Hs$ is infinite-dimensional, the variation of the data may concentrate in relatively few directions. The covariance operator of the data, pooled over both groups, ranks these directions by variance:
\begin{equation}
  \Sigma \;:=\; \E_{Q}\bigl[(\phi(X) - \mu_Q) \otimes (\phi(X) - \mu_Q)\bigr],
  \qquad
  Q \;:=\; \tfrac{1}{2}(P_a + P_b).
  \label{eq:covariance}
\end{equation}
For bounded $k$, the eigendecomposition of $\Sigma$ gives eigenvalues $\lambda_1 \ge \lambda_2 \ge \cdots \ge 0$ and orthonormal eigenvectors $e_1, e_2, \ldots$ on its positive spectrum. We first use a standard source condition \citep{caponnetto2007optimal}:

\noindent\textbf{(A1) Source condition.} $\delta = \Sigma^r u$ with $\normH{u} \le R$, for fixed $r, R > 0$.

This limits how much of the difference between the groups can lie in directions of small variance; with $r = 1/2$ it holds for every pair of groups (\cref{app:source}). These vectors form the ellipsoid $B_r(R) := \{\Sigma^r u : \normH{u} \le R\}$, with semi-axes $R \lambda_j^r$. The supremum below ranges over every vector in this ellipsoid (\cref{app:source}).

\begin{theorem}[Optimal residual]
\label{thm:pokemon-quant}
Under \textbf{(A1)},
\begin{equation}
  \inf_{\dim V \le m}\; \sup_{\delta \in B_r(R),\ \delta \perp V}\; \normH{\delta}
  \;=\; R \lambda_{m+1}^r,
  \label{eq:pokemon-quant}
\end{equation}
attained by the $m$ optimal criteria $e_1, \ldots, e_m$.
\end{theorem}

\begin{proof}[Proof sketch]
Satisfying the first $m$ eigendirections removes the $m$ longest axes of the ellipsoid, leaving a residual of at most the next longest semi-axis, $R \lambda_{m+1}^r$. By pigeonhole, any $m$-dimensional subspace leaves a direction in the span of the first $m + 1$ eigenvectors unchecked. The ellipsoid extends at least $R \lambda_{m+1}^r$ along this direction, giving the matching lower bound (\cref{app:pokemon-quant}).
\end{proof}

Thus, satisfying the best $m$ criteria restricts the worst-case residual unfairness to $R \lambda_{m+1}^r$. To determine how quickly it decays as we add criteria, we make a second standard assumption:

\noindent\textbf{(A2) Polynomial eigendecay.} $\lambda_j \asymp j^{-\alpha}$ for some $\alpha > 1$.

\begin{corollary}[Rate with more criteria]
\label{cor:pokemon-rate}
Under \textbf{(A1)} and \textbf{(A2)}, the optimal residual in~\eqref{eq:pokemon-quant} satisfies $R \lambda_{m+1}^r = \Theta(R\, m^{-\alpha r})$.
\end{corollary}

\paragraph{Implications.}
The optimal residual is the Kolmogorov $m$-width of the source class \citep{pinkus1985widths}. As we add criteria, the residual unfairness decays at rate $m^{-\alpha r}$; faster spectral decay (larger $\alpha$) or a smoother difference between the groups (larger $r$) makes it decay faster. Nevertheless, the worst case remains nonzero for every finite $m$, even when all criteria are satisfied.

\subsection{Other Criteria}
\label{sec:pokemon-other}

\paragraph{Inexact criteria.}
\label{sec:pokemon-inexact}
Practical criteria hold only up to a tolerance, $|\inner{v_i}{\delta}_{\Hs}| \le \varepsilon$. These readings constrain only the part of the difference that the checklist sees, $P_V \delta$, where $P_V$ is the orthogonal projection onto $V$. For orthonormal criteria, Pythagoras gives
\begin{equation}
  \normH{P_{V^\perp} \delta}^2
  \;=\; \MMD^2(P_a, P_b) - \normH{P_V \delta}^2
  \;\ge\; \MMD^2(P_a, P_b) - m \varepsilon^2.
  \label{eq:residual-decomposition}
\end{equation}
Thus, if $\MMD(P_a, P_b) > \sqrt{m}\,\varepsilon$, a criterion orthogonal to the checklist is still violated (\cref{app:inexact}).

\paragraph{Different distributions.}
Criteria can also be linear in distributions other than that of $X$:
\begin{itemize}[leftmargin=1.5em, itemsep=1pt, topsep=2pt]
  \item equal error rates compare the distribution of $X$ conditioned on the outcome $Y = y$ in the two groups;
  \item calibration of a score $S = s(X)$ within each group, $\E[(Y - S)\, h(S) \mid G = g] = 0$ for every bounded function $h$, involves the joint distribution of $(X, Y)$.
\end{itemize}
The Pok\'emon theorem applies to each, with $\delta$ replaced by the conditional mean discrepancy of the corresponding distributions (\cref{app:dictionary}).

\paragraph{Nonlinear criteria.}
Criteria such as predictive parity, which compares a ratio of two expectations, are nonlinear in the distribution. A von Mises expansion, a Taylor expansion in the distribution, writes such a criterion $T$ as a series in differences of expectations $\Delta_j := \E[c_j \mid G = a] - \E[c_j \mid G = b]$, namely $T(P_a) - T(P_b) = \sum_j a_j \Delta_j + \frac{1}{2} \sum_{j, k} b_{jk} \Delta_j \Delta_k + \cdots$ \citep{fernholz1983von}. Each $\Delta_j$ is a linear criterion. If finitely many appear, they form a finite checklist, and satisfying it makes every term zero even when the groups differ.

\paragraph{Infinite criteria.}
A checklist may also contain infinitely many criteria. If their span is not dense in $\Hs$, an orthogonal complement remains, and any difference between the groups that lies there satisfies every criterion. (Think of Hilbert's hotel: checking every other eigendirection still leaves infinitely many unchecked.) Calibration is such a family: it has infinitely many criteria, one for every bounded function $h$, yet each depends only on the score and the outcome, so groups whose features differ beyond the score satisfy all of them (\cref{app:infinite}). A dense infinite checklist does certify; \cref{sec:recover} shows where requiring one leads.

\section{Recovering Classical Impossibilities}
\label{sec:recover}

\Cref{sec:pokemon} showed that a finite checklist cannot certify fairness, and that an infinite checklist that is not dense can leave differences between the groups unchecked. A dense infinite checklist leaves none. \emph{What happens if we use a dense infinite checklist to certify fairness?} We show that it certifies one of the three traditional notions of fairness, each of which can hold on its own. When base rates differ, however, no useful score (other than a perfect one) satisfies any two of them at once. This recovers the classical impossibility theorems, including KMR from the introduction.

\subsection{Three Notions of Fairness}
\label{sec:recover-notions}

The classical results concern a score $S \in [0, 1]$ computed from the features $X$, such as the risk score of COMPAS. Each fairness criterion uses one of three distributions of the score: its overall distribution, its distribution conditioned on the outcome, or its joint distribution with the outcome. Write $\mu_g$ for the mean embedding of the score's distribution in group $g$, $\mu_{y, g}$ for that of its distribution in group $g$ conditioned on $Y = y$, and $\delta := \mu_a - \mu_b$ and $\delta_y := \mu_{y, a} - \mu_{y, b}$, as in \cref{sec:criteria-directions} with $S$ in place of $X$. A dense infinite checklist for one of the three distributions certifies the corresponding traditional notion of fairness \citep{barocas2023fairness}:
\begin{enumerate}[label=(C\arabic*), leftmargin=2.5em, itemsep=1pt, topsep=2pt]
  \item \emph{Independence}: $S$ has the same distribution in both groups, $\delta = 0$. The dense checklist of equal selection rates $\Prb[S > t \mid G = g]$ at every threshold $t$ certifies it.
  \item \emph{Separation}: $S$ has the same distribution in both groups conditioned on each outcome, $\delta_0 = \delta_1 = 0$. The dense checklist of equal error rates $\Prb[S > t \mid Y = y, G = g]$ at every threshold $t$ certifies it.
  \item \emph{Sufficiency}: a score means the same in both groups. Among individuals with score $s$, the fraction of positive outcomes is the same in both groups, $f(s)$. The checklist $\E[(Y - f(S))\, h(S) \mid G = g] = 0$, for every bounded function $h$ of the score, certifies it. These criteria use the joint distribution of $(S, Y)$ and are checked within each group. Calibration is the case $f(s) = s$.
\end{enumerate}
These checklists are dense among functions of the score, though not among functions of the features (\cref{sec:pokemon-other}), so they certify properties of the score only. When $S$ can take infinitely many values, the Pok\'emon theorem shows that finitely many thresholds certify neither independence nor separation (\cref{lem:realizability}).

Each notion can hold on its own. A constant score satisfies independence (C1), and a perfect score, $S = Y$, satisfies separation (C2) and sufficiency (C3). Both are degenerate: a constant score carries no predictive signal, and a perfect score requires an outcome that can be predicted without error. We call a score useful if it carries predictive signal. \emph{Can a useful score (other than a perfect one) satisfy any two of the notions at once?}

\subsection{Why Any Two Notions Conflict}
\label{sec:recover-exact}

When base rates differ, no useful score (other than a perfect one) satisfies any two of the notions at once:

\begin{theorem}[Any two notions]
\label{thm:recover}
Suppose the base rates differ, $p_a \ne p_b$.
\begin{enumerate}[label=(\roman*), leftmargin=2em, itemsep=1pt, topsep=2pt]
  \item A score that satisfies independence (C1) and separation (C2) carries no predictive signal (the score is independent of $Y$).
  \item Only a perfect score satisfies separation (C2) and sufficiency (C3): $f(S) = Y$ almost surely, and $S = Y$ if $S$ is calibrated.
  \item No score satisfies independence (C1) and sufficiency (C3).
\end{enumerate}
\end{theorem}

\begin{proof}[Proof sketch]
The law of total expectation connects the three distributions of the score: $\mu_g = p_g\, \mu_{1, g} + (1 - p_g)\, \mu_{0, g}$. Subtracting the groups gives
\begin{equation}
  \delta \;=\; p_a\, \delta_1 + (1 - p_a)\, \delta_0 + (p_a - p_b)\, \omega,
  \qquad
  \omega \;:=\; \mu_{1, b} - \mu_{0, b},
  \label{eq:coupling}
\end{equation}
where $\omega$, the predictive signal, is the difference between the score's distributions in the two outcome classes of group $b$.

(i): Independence (C1) asks for $\delta = 0$, and separation (C2) for $\delta_0 = \delta_1 = 0$. Together they leave $(p_a - p_b)\, \omega = 0$ in~\eqref{eq:coupling}, so $\omega = 0$; thus the score carries no predictive signal.

(ii) and (iii): Since $f(s)$ is the fraction of positive outcomes at score $s$, sufficiency (C3) fixes each group's average of $f(S)$ at its base rate, $\E[f(S) \mid G = g] = p_g$. (iii) then follows, as independence (C1) makes these averages equal, which requires $p_a = p_b$. Under separation (C2), the same law of total expectation gives $p_a - p_b = (p_a - p_b)\,(\E[f(S) \mid Y = 1] - \E[f(S) \mid Y = 0])$. Since $f(S) \in [0, 1]$, this forces $f(S) = Y$, which proves~(ii) (\cref{app:recover}).
\end{proof}

The proofs of~(ii) and~(iii) use only averages of $f(S)$. (ii) becomes KMR for a calibrated score, $f(s) = s$. KMR requires only equal average scores within each outcome class instead of separation:

\begin{theorem}[KMR, \citealp{kleinberg2017inherent}]
\label{thm:kmr}
Suppose $p_a \ne p_b$. If $S$ is calibrated within each group and has equal average scores in both groups within each outcome class, then $S = Y$ almost surely.
\end{theorem}

Statements~(i) and~(iii) are incompatibilities identified by \citet{barocas2023fairness}, and~(ii) strengthens their result for separation and sufficiency (\cref{app:recover}). \citet{stewart2024impossibility} weaken calibration in KMR to equal gaps between average score and base rate in the two groups; our proof of~(ii) covers this case. The result of \citet{chouldechova2017fair}, that a binary classifier generally cannot have both equal error rates and equal positive predictive values, follows in the same way (\cref{app:recover}).

\paragraph{Finitely many criteria.}
As in \cref{sec:pokemon-quantitative}, the best $m$ criteria of a notion are the leading eigendirections of the covariance operator of the corresponding distribution of the score, pooled over both groups. With only these, the conclusions of \cref{thm:recover} hold approximately, with an error that shrinks as $m$ grows. For example, a calibrated score that satisfies the best $m$ criteria of separation has
\[
  \E\bigl[|S - Y| \bigm| G = g\bigr] \;\le\; \frac{1}{\sqrt{m + 1}\, |p_a - p_b|}.
\]
This follows from \cref{thm:pokemon-quant} with the characteristic kernel $\min(s, s')$ on scores (\cref{thm:recover-approx}). In this way, the Pok\'emon theorem gives an approximate KMR: the more of these criteria a calibrated score satisfies, the closer it must be to a perfect score.

\paragraph{Implications.}
The Pok\'emon theorem is an \emph{insufficiency}: a finite checklist of satisfied criteria certifies nothing. The classical results are \emph{incompatibilities}: certain criteria cannot hold together. The insufficiency leads to the incompatibilities, and in this sense, the Pok\'emon theorem is the source of the classical impossibilities. Escaping the Pok\'emon theorem requires dense checklists, and when base rates differ, certifying any two notions with them forces a score without predictive signal or a perfect score, or allows no score at all.

\section{The Impossibility of Fair Feature Learning}
\label{sec:fair-feature}

We now turn to the second question of the introduction: \emph{can we learn features that scrub group information, yet retain predictive signal?} Since unfairness originates in the data rather than in the algorithm, a natural remedy is to change the data. Fair representation learning learns an encoder $\Phi : \Xs \to \Zs$ whose features $Z = \Phi(X)$ scrub the sensitive group information, so that any model built on them inherits fairness for free. Many different methods pursue this goal, using prototypes, adversaries, variational autoencoders, and MMD penalties \citep{zemel2013learning,edwards2016censoring,louizos2016variational,madras2018learning,deka2023mmdbfair,jovanovic2023fare}. \emph{Why are there so many methods?} We show that when base rates differ, it is impossible to learn truly fair features that retain predictive signal.

\paragraph{What fair features need.}
For any model built on $Z$ to inherit fairness, the distribution of $Z$ must be insensitive to the group: $Z$ must satisfy independence (C1) (often called parity), the same distribution in both groups, and separation (C2), the same distribution in both groups within each outcome class. To be useful, $Z$ must also retain predictive signal: different distributions in the two outcome classes.

\begin{theorem}[Impossibility of fair feature learning]
\label{thm:fair-feature}
Suppose the base rates differ, $p_a \ne p_b$. If the features $Z = \Phi(X)$ satisfy independence (C1) and separation (C2), they carry no predictive signal.
\end{theorem}

\begin{proof}[Proof sketch]
This is \cref{thm:recover}(i) with $Z$ in place of $S$ (\cref{app:fair-feature-proof}).
\end{proof}

\paragraph{Implications.}
This closes the door on fair feature learning as usually posed: when base rates differ, features that make every model built on them fair carry no predictive signal. It also explains the variety of methods: each approximates, in its own way, a goal that no method can reach. Useful features must give up some independence or some separation. \emph{How much predictive signal can features keep if they give up only a little fairness?} \Cref{sec:approx} answers this question.

\section{Approximate Fairness}
\label{sec:approx}

Features that give up some fairness can retain predictive signal, but only in proportion to the unfairness they permit. We prove this bound and give a simple rule that attains it.

We write $\mu_g$, $\mu_{y, g}$, $\delta$ and $\delta_y$ as in \cref{sec:recover}, with the features $Z$ in place of the score $S$ and a bounded characteristic kernel on $\Zs$. We measure how far the features are from independence (C1) by $\varepsilon := \normH{\delta}$, and how far they are from separation (C2) by $\rho := \max_y \normH{\delta_y}$. Their predictive signal in group $g$ is $\omega_g := \mu_{1, g} - \mu_{0, g}$, the difference between their distributions in the two outcome classes, so that $\omega_b$ is the $\omega$ of~\eqref{eq:coupling}.

\begin{theorem}[Fairness and predictive signal]
\label{thm:approx-fair-feature}
Suppose $p_a \ne p_b$. In each group $g$,
\begin{equation}
  \normH{\omega_g} \;\le\; \frac{\varepsilon + \rho}{|p_a - p_b|}.
  \label{eq:approx-pareto}
\end{equation}
\end{theorem}

\begin{proof}[Proof sketch]
By~\eqref{eq:coupling} with $Z$ in place of $S$, $(p_a - p_b)\, \omega_b = \delta - p_a \delta_1 - (1 - p_a)\, \delta_0$. Taking norms gives $|p_a - p_b|\, \normH{\omega_b} \le \varepsilon + \rho$, and exchanging the groups gives the bound for $\omega_a$ (\cref{app:approx-identity}).
\end{proof}

\paragraph{Implications.}
To keep predictive signal of size $\normH{\omega_g}$, features must allow unfairness $\varepsilon + \rho$ of at least $|p_a - p_b| \cdot \normH{\omega_g}$. The more the base rates differ, the more fairness must be given up for the same signal. \Cref{thm:fair-feature} is the case $\varepsilon = \rho = 0$, where the bound forces the predictive signal to zero.

This bound limits every encoder. \emph{Can we attain this bound?} Let $\|\omega\|_{\mathrm{sep}}$ be the supremum of $\normH{\omega_g}$ over features that satisfy separation (C2).

\begin{proposition}[The bound is attained]
\label{prop:frontier}
Suppose $p_a \ne p_b$. For every $\beta < |p_a - p_b|\, \|\omega\|_{\mathrm{sep}}$, there exist (randomized) features with $\rho = 0$ and $\varepsilon = \beta$ whose predictive signal is $\normH{\omega_a} = \normH{\omega_b} = \beta / |p_a - p_b|$.
\end{proposition}

\begin{proof}[Proof sketch]
Features that satisfy separation ($\rho = 0$) lie exactly on the bound: \eqref{eq:coupling} leaves $\delta = (p_a - p_b)\, \omega_b$, so $\varepsilon = |p_a - p_b|\, \normH{\omega_g}$. To lower their independence violation $\varepsilon$, replace the features by a constant value $z_0$ for a random fraction $1 - \lambda$ of individuals. This shrinks every difference between their distributions by the factor $\lambda$, so $\varepsilon$ and the signal $\normH{\omega_g}$ shrink together, $\rho$ stays zero, and the features stay on the bound. Starting from separated features with signal close to $\|\omega\|_{\mathrm{sep}}$ reaches every point below it (\cref{app:approx-identity}).
\end{proof}

\paragraph{Implications.}
Up to $\|\omega\|_{\mathrm{sep}}$, the Pareto frontier between unfairness $\varepsilon + \rho$ and predictive signal $\normH{\omega_g}$ is the straight line of \cref{thm:approx-fair-feature}, and the construction above attains every point on it. Beyond $\|\omega\|_{\mathrm{sep}}$, more signal requires giving up separation, and the frontier depends on the data.

\paragraph{Computing the frontier.}
For a binary decision, these quantities are familiar rates, up to a common constant: $\varepsilon$ is the gap in selection rates between the groups, $\rho$ is the larger of the gaps in true- and false-positive rates, and $\normH{\omega_g}$ is $|\TPR_g - \FPR_g|$.

For decisions based on a score, the frontier is then the solution of a linear program:
\begin{enumerate}[label=(\arabic*), leftmargin=2em, itemsep=1pt, topsep=2pt]
  \item Estimate $\Prb[Y = 1 \mid X]$, with scores computed on rows the model was not fit to.
  \item In each group $g$, predict a positive outcome, $\hat Y = 1$, when the score exceeds a threshold $t$, where $t$ is drawn at random, with probabilities $w_{g, t}$, from the values the score takes on group $g$'s held-out data. The selection rates, $\TPR_g$ and $\FPR_g$ are then averages of the rates at each threshold, so they are linear in the weights $w_{g, t}$.
  \item For each budget $\beta$, maximize $\TPR_g - \FPR_g$ subject to $\varepsilon + \rho \le \beta$.
\end{enumerate}
With $\rho = 0$ imposed, the program gives the most informative separated decision, at $\|\omega\|_{\mathrm{sep}}$, and replacing it by a constant decision for a random fraction of individuals traces the straight line below it. With $\rho > 0$ allowed, larger budgets trace the rest of the frontier, which bends below the line beyond $\|\omega\|_{\mathrm{sep}}$. \Cref{sec:experiments} computes this frontier on three datasets and compares common fairness methods with it. \citet{hardt2016equality} use the same rules, a separate randomized threshold in each group, to minimize error under equal error rates.

\section{Experiments}
\label{sec:experiments}

We measure the residual left by the best $m$ criteria (\cref{thm:pokemon-quant}), how close fair feature methods come to the bound of \cref{thm:approx-fair-feature}, and the frontier of \cref{sec:approx} on three datasets: Adult Income (men and women), COMPAS (Black and white defendants) and ACS PUMS income data from California (white and non-white respondents). \Cref{app:experiments} gives the methods and further results.

\begin{figure}[p]
  \centering
  \includegraphics[width=\textwidth]{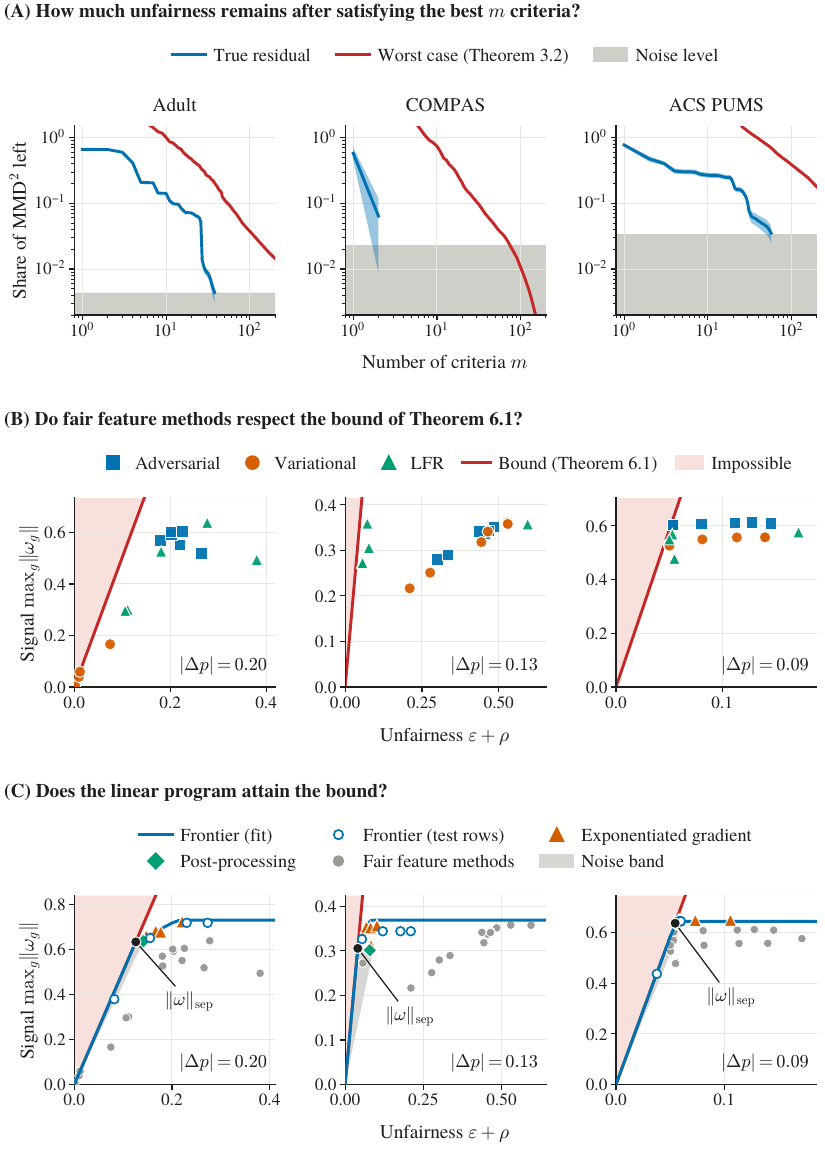}
  \caption{\textbf{(A)} True residual unfairness after the best $m$ criteria (blue) and the worst-case residual of \cref{thm:pokemon-quant} (red). \textbf{(B)} Predictive signal vs.\ unfairness of fair feature methods, one point per setting of their fairness weight. \textbf{(C)} The frontier from the linear program, fit on training and validation rows (line) and measured on test rows (circles). In (B) and (C), the red line is the bound of \cref{thm:approx-fair-feature}, and no rule can lie in the shaded region above it. Test-row values are averaged over ten random splits of each dataset, and blue bands are 95\% intervals.}
  \label{fig:experiments}
\end{figure}

\paragraph{Residual after the best $m$ criteria.}
\emph{How much unfairness remains after satisfying the best $m$ criteria?} \Cref{fig:experiments}A takes the best $m$ criteria (the leading eigendirections of the pooled covariance $\Sigma$, \cref{thm:pokemon-quant}) and measures the true residual unfairness: the share of $\MMD^2(P_a, P_b)$ that these criteria do not see. The red line is the worst-case residual presented in \cref{thm:pokemon-quant}, computed from the estimated eigenvalues of $\Sigma$ with $r = 1/2$ and $R = 2$. This source condition holds for any two groups (\cref{prop:source-half}), and we see the worst case (red line) bound the true residual (blue line) across all three datasets. We can also see that the worst-case lines are roughly straight on the log--log plot, which is the case under polynomial eigendecay \textbf{(A2)}. Residuals inside the gray band are too small to measure. On COMPAS, the residual falls into this band after only two criteria.

\paragraph{Fair features.}
\emph{Do fair feature methods respect the bound of \cref{thm:approx-fair-feature}?} \Cref{thm:approx-fair-feature} says that features can gain predictive signal only by giving up fairness: moving up in \cref{fig:experiments}B (more predictive signal) requires moving right (more unfairness). The red line draws this bound, $\max_g \normH{\omega_g} = (\varepsilon + \rho)/|p_a - p_b|$, and no features can lie above it. Its slope is steepest on ACS PUMS, whose base rates differ least. We train adversarial \citep{edwards2016censoring}, variational \citep{louizos2016variational} and LFR \citep{zemel2013learning} fair feature encoders, each with a weight that trades fairness against accuracy. This weight scales the adversary's loss, an MMD penalty between the groups, or LFR's parity term. For each value of the weight, \cref{fig:experiments}B plots the predictive signal that the features retain against the unfairness they allow, measured on a classifier trained on the features. As \cref{thm:approx-fair-feature} guarantees, every point lies below the red line; most lie well below it.

\paragraph{Computing the frontier.}
\emph{Does the linear program attain the bound?} \Cref{fig:experiments}C computes the frontier with the linear program from \cref{sec:approx}, on scores from gradient-boosted trees \citep{chen2016xgboost} or a neural network, whichever predicts better on validation data. We fit the program on training and validation rows (line), and measure the same rules on test rows (circles). The frontier follows the bound (red line) up to the black point $\|\omega\|_{\mathrm{sep}}$, the most signal a rule keeps with exact separation ($\rho = 0$), as \cref{prop:frontier} guarantees. Beyond it, more signal requires giving up separation. The gray band shows where a rule exactly on the bound would land on the test rows by chance. Exact separation costs the group that is harder to predict nothing: on all 30 pairs of dataset and split, $\|\omega\|_{\mathrm{sep}}$ equals the most signal any threshold rule on the score achieves in that group. Beyond the black point, the extra signal comes only from the group that is easier to predict. Exponentiated gradient \citep{agarwal2018reductions}, optimizing the same objective, lands near the frontier at tight tolerances, and post-processing \citep{hardt2016equality} solves the same program. The gray points repeat the fair feature methods of panel B.

\section{Discussion}
\label{sec:discussion}

\paragraph{Scope and open problems.}
We consider two groups and a binary outcome; individual and causal notions of fairness lie outside this framework. A finite-dimensional class of criteria, such as all linear combinations of $d$ features, can be certified by $d$ criteria; the Pok\'emon theorem concerns equality of the groups' distributions (\cref{app:realizability}). Learned features are often reused for tasks whose outcomes are unknown when the features are learned. Separation depends on each task's outcome, and no features that allow accurate predictions are fair for every task \citep{lechner2021impossibility}; which features are best for a family of tasks remains open.

\paragraph{Fairness as budget allocation.}
Exact fairness is impossible, but that does not let algorithm designers off the hook. The difference $\delta$ between the groups acts as an unfairness budget. A designer decides how to spend it: on predictive signal, or on fairness criteria. The two compete. By the law of total expectation~\eqref{eq:coupling},
\[
  \delta \;=\; p_a\, \delta_1 + (1 - p_a)\, \delta_0 + (p_a - p_b)\, \omega,
\]
so the budget splits between differences within each outcome class and the predictive signal $\omega$. Keeping signal of size $\normH{\omega}$ therefore leaves unfairness $\varepsilon + \rho$ of at least $|p_a - p_b| \cdot \normH{\omega}$ (\cref{thm:approx-fair-feature}). Which criteria to spend it on depends on the social setting: equal error rates may matter most in criminal justice, and equal selection rates in hiring. Following a checklist blindly can use up the budget on criteria that matter little in the application. A fair model needs more than a satisfied checklist: it needs criteria chosen deliberately for the people it affects.

\newpage

\begin{ack}
We thank Yu-Xiang Wang for suggesting that a connection between the Pok\'emon theorem and KMR should exist, and Ryan Tibshirani for helpful comments and feedback on the manuscript. This work received no external funding. The authors declare no competing interests.
\end{ack}

\bibliographystyle{plainnat}
\bibliography{references}

\appendix
\crefalias{section}{appendix}
\crefalias{subsection}{subappendix}
\newpage

\section{Kernels in Brief}
\label{app:kernels}

This appendix explains the kernel tools of \cref{sec:kernels-embeddings} for readers who are new to them; \citet{scholkopf2002learning} and \citet{gretton2012kernel} give full treatments.

\paragraph{Feature map and reproducing property.}
A kernel $k(x, x')$ measures how similar two individuals are. The Gaussian kernel $k(x, x') = \exp(-\|x - x'\|^2 / (2\sigma^2))$, for example, is close to 1 when $x$ and $x'$ are close and decays to 0 as they move apart. A positive-definite kernel is an inner product between features, $k(x, x') = \inner{\phi(x)}{\phi(x')}_{\Hs}$, where the feature map $\phi$ takes values in a Hilbert space $\Hs$, the RKHS. The elements of $\Hs$ are functions on $\Xs$, and evaluating one is again an inner product, $f(x) = \inner{f}{\phi(x)}_{\Hs}$. This reproducing property makes every test function $f \in \Hs$ a vector in the same space as the features. For the Gaussian kernel, $\Hs$ is infinite-dimensional, and its functions approximate every continuous function on a compact set.

\paragraph{Mean embeddings and characteristic kernels.}
Averaging the features over a distribution $P$ gives its mean embedding $\mu_P = \E_P[\phi(X)]$, a single vector that represents $P$, and the expectation of any $f \in \Hs$ is $\E_P[f(X)] = \inner{f}{\mu_P}_{\Hs}$. How much of $P$ the mean embedding records depends on the kernel. For the linear kernel $k(x, x') = x^\top x'$, $\phi(x) = x$, and $\mu_P$ is the ordinary mean, which many distributions share. A kernel is characteristic when the mean embedding records the whole distribution: $\mu_P = \mu_Q$ only if $P = Q$. The Gaussian, Laplace and Mat\'ern kernels on $\R^d$ are characteristic. With a characteristic kernel, the difference $\delta = \mu_a - \mu_b$ is zero exactly when the groups' distributions are equal.

\paragraph{MMD and the witness.}
The MMD is the length of this difference, $\MMD(P, Q) = \normH{\mu_P - \mu_Q}$. By the Cauchy--Schwarz inequality, it is the largest difference in expectation that a test function with $\normH{f} \le 1$ can detect, and the witness, the unit vector along $\mu_P - \mu_Q$, detects it. The MMD can be estimated from kernel evaluations alone. The unbiased estimate of $\MMD^2(P, Q)$ averages $k$ over pairs of distinct points within each sample, adds the two averages, and subtracts twice the average over pairs drawn across the two samples \citep{gretton2012kernel}.

\paragraph{Covariance operator.}
For a distribution $Q$, the covariance operator $\Sigma = \E_Q[(\phi(X) - \mu_Q) \otimes (\phi(X) - \mu_Q)]$, with $(f \otimes f')\, v := \inner{f'}{v}_{\Hs}\, f$, generalizes the covariance matrix: for a test function $v \in \Hs$, $\inner{v}{\Sigma v}_{\Hs}$ is the variance of $v(X)$ under $Q$. For a bounded kernel, $\Sigma$ has eigenvalues $\lambda_1 \ge \lambda_2 \ge \cdots \ge 0$ with a finite sum, and orthonormal eigenvectors $e_1, e_2, \ldots$. Among unit test functions orthogonal to $e_1, \ldots, e_{j-1}$, the eigenvector $e_j$ has the largest variance, $\lambda_j$. Since the eigenvalues have a finite sum, they decay, and a few directions carry most of the variation. On $n$ samples, they are the eigenvalues of the centered kernel matrix divided by $n$, as in kernel principal component analysis \citep{scholkopf2002learning}.

\section{Proofs for \texorpdfstring{\cref{sec:pokemon}}{Section 3}}
\label{app:proofs-pokemon}

This appendix proves the results of \cref{sec:pokemon}. Throughout, $k$ is a bounded measurable kernel with separable RKHS $\Hs$, so that mean embeddings and covariance operators exist, and $\kappa^2 := \sup_x k(x, x)$.

\subsection{Groups That Satisfy a Checklist yet Differ}
\label{app:realizability}

\Cref{thm:pokemon} assumes that the groups differ. The lemma below shows that, for every finite checklist, there are two groups that satisfy it yet differ, by an amount we can choose. It reweights one group by a function that every criterion misses: $m$ criteria and the total probability impose only $m + 1$ conditions, so a reweighting with $m + 2$ free values exists.

\begin{lemma}[Realizability]
\label{lem:realizability}
Let $k$ be a bounded characteristic kernel on a separable metric space $\Xs$, and let $P$ be a distribution on $\Xs$ whose support contains infinitely many points. For every finite checklist $v_1, \ldots, v_m \in \Hs$, there is $\gamma_{\max} > 0$ such that, for every $\gamma \in (0, \gamma_{\max}]$, there are group distributions $P_a = P$ and $P_b \ne P_a$ that satisfy every criterion, $\inner{v_i}{\delta}_{\Hs} = 0$, and have $\MMD(P_a, P_b) = \gamma$. The distribution $P_b$ has a bounded, nonnegative density with respect to $P$.
\end{lemma}

\begin{proof}
The criteria are bounded functions, since $|v_i(x)| = |\inner{v_i}{\phi(x)}_{\Hs}| \le \kappa \normH{v_i}$.

\emph{A reweighting that every criterion misses.} Choose $m + 2$ distinct points in the support of $P$, and disjoint open balls $A_1, \ldots, A_{m+2}$ around them; each ball has $P(A_j) > 0$. A step function $h = \sum_j \theta_j \Ind_{A_j}$ must satisfy $\int h \, dP = 0$ and $\int v_i\, h \, dP = 0$ for $i = 1, \ldots, m$. These are $m + 1$ homogeneous linear equations in the $m + 2$ unknowns $\theta_j$, so they have a nonzero solution.

\emph{The reweighted groups satisfy every criterion.} For $0 < t \le 1 / \|h\|_\infty$, let $P_a := P$ and $dP_b := (1 + t h)\, dP$. Since $1 + t h \ge 0$ and $\int h\, dP = 0$, $P_b$ is a probability distribution, and $P_b \ne P_a$ because $h \ne 0$ on a set of positive probability. The difference between the groups is $\delta = -t\, g_h$ with $g_h := \int \phi\, h \, dP$, so $\inner{v_i}{\delta}_{\Hs} = -t \int v_i\, h\, dP = 0$ for every $i$.

\emph{The difference is nonzero and can have any size up to $\gamma_{\max}$.} Write $h = h^+ - h^-$ for its positive and negative parts. Since $\int h\, dP = 0$, both parts have the same positive mass, and dividing by it turns $h^+ dP$ and $h^- dP$ into probability distributions $P^+$ and $P^-$. They live on disjoint sets, so they differ, and since $k$ is characteristic, $g_h$, a positive multiple of $\mu_{P^+} - \mu_{P^-}$, is nonzero. Hence $\MMD(P_a, P_b) = t \normH{g_h}$ takes every value in $(0, \gamma_{\max}]$ as $t$ ranges over $(0, 1/\|h\|_\infty]$, with $\gamma_{\max} := \normH{g_h} / \|h\|_\infty$.
\end{proof}

\paragraph{Scope of the lemma.}
The kernel enters only to show that the groups differ in MMD; the criteria need only be bounded. The lemma therefore also holds for test functions outside $\Hs$, such as a classifier's decisions or the threshold indicators $\Ind[s > t]$ of a score, and within an outcome class, by reweighting the distribution of $X$ given $Y = y$ in one group. Applied to a score that takes infinitely many values, it shows that finitely many thresholds certify neither independence (C1) nor separation (C2) (\cref{sec:recover-notions}). The assumption of infinite support cannot be dropped. If $P$ takes only $d$ values, $d - 1$ criteria that, together with the constant function, span all functions on these values certify $P_a = P_b$ among distributions on them. Likewise, a kernel with a $D$-dimensional feature map, such as the linear kernel on $\R^D$, is not characteristic when $\Xs$ has more than $D + 1$ points. Then $D$ criteria that span its directions certify only that $D$ expectations agree: the construction in the proof, applied to these $D$ coordinates, gives groups that agree on all of them yet differ.

\paragraph{What a checklist certifies.}
Certifying a class of criteria differs from certifying that the groups have equal distributions. A finite-dimensional class is certified by criteria that span it; the $d$ averages of the input features, for example, certify every criterion that is linear in the features. Equal distributions cannot be certified by finitely many criteria when the data take infinitely many values, and this is what the Pok\'emon theorem concerns (\cref{sec:discussion}). \Cref{app:experiments-finite} measures both on the three datasets. The witness of \cref{thm:pokemon} shows that the groups' distributions differ; whether the difference matters for a particular decision depends on the application.

\subsection{Proof of \texorpdfstring{\cref{thm:pokemon-quant} and \cref{cor:pokemon-rate}}{Theorem 3.2 and Corollary 3.3}}
\label{app:pokemon-quant}

The proof computes the residual left by the leading eigendirections, then shows that no other $m$ criteria leave less. For a bounded kernel and separable $\Hs$, $\Sigma$ is positive and trace-class. Its eigenvectors $e_j$ with positive eigenvalues span the closure of its range, which contains the ellipsoid $B_r(R)$. If $\Sigma$ has rank at most $m$, then $\lambda_{m+1} = 0$, and criteria that span its range force $\delta = 0$, so both sides of~\eqref{eq:pokemon-quant} vanish. Otherwise $\lambda_{m+1} > 0$.

\emph{Upper bound.} Let $V_m := \spanop\{e_1, \ldots, e_m\}$, and let $\delta = \Sigma^r u \in B_r(R)$ satisfy every criterion in $V_m$. With $u_j := \inner{u}{e_j}_{\Hs}$, the criteria give $\inner{e_j}{\delta}_{\Hs} = \lambda_j^r u_j = 0$ for $j \le m$, so
\[
  \normH{\delta}^2 \;=\; \sum_{j > m} \lambda_j^{2r} u_j^2 \;\le\; \lambda_{m+1}^{2r} \normH{u}^2 \;\le\; R^2 \lambda_{m+1}^{2r},
\]
with equality at $\delta = R \lambda_{m+1}^r e_{m+1}$.

\emph{Lower bound.} Let $V$ be any subspace with $\dim V \le m$. The span of $e_1, \ldots, e_{m+1}$ has dimension $m + 1$, so it contains a unit vector $w = \sum_{j \le m+1} w_j e_j$ orthogonal to $V$. The vector $\delta^\star := R \lambda_{m+1}^r w$ satisfies every criterion in $V$ and has norm $R \lambda_{m+1}^r$. It lies in $B_r(R)$, since $\delta^\star = \Sigma^r u$ with $u = R \sum_{j \le m+1} w_j (\lambda_{m+1} / \lambda_j)^r e_j$, and
\[
  \normH{u}^2 \;=\; R^2 \sum_{j \le m+1} w_j^2 \Bigl(\frac{\lambda_{m+1}}{\lambda_j}\Bigr)^{2r} \;\le\; R^2.
\]
The two bounds prove~\eqref{eq:pokemon-quant}. Dropping the requirement $\delta \perp V$ and measuring the part of $\delta$ orthogonal to $V$ gives the same value, which is the Kolmogorov $m$-width of the ellipsoid \citep{pinkus1985widths}. Under \textbf{(A2)}, $\lambda_{m+1}^r \asymp (m + 1)^{-\alpha r} \asymp m^{-\alpha r}$, which proves \cref{cor:pokemon-rate}.

\subsection{When the Source Condition Holds}
\label{app:source}

The source condition \textbf{(A1)} limits how much of the difference between the groups can lie in directions of small variance. With $r = 1/2$, it holds for every pair of groups, with a radius set by how different they are. Let $h := dP_a/dQ - dP_b/dQ$ be the difference between the groups' densities with respect to their average $Q = \tfrac12(P_a + P_b)$. Each density lies between 0 and 2, because $Q \ge \tfrac12 P_g$, so $|h| \le 2$.

\begin{proposition}[The source condition holds with $r = 1/2$]
\label{prop:source-half}
The difference between the groups is $\delta = \Sigma^{1/2} u$ with $\normH{u} \le \|h\|_{L^2(Q)}$, and
\[
  \|h\|_{L^2(Q)}^2 \;=\; 2 \int \frac{(dP_a - dP_b)^2}{d(P_a + P_b)} \;\le\; 4\, \mathrm{TV}(P_a, P_b) \;\le\; 4,
\]
where $\mathrm{TV}$ is the total variation distance. Hence \textbf{(A1)} holds with $r = 1/2$ and $R = 2$ for every pair of groups.
\end{proposition}

\begin{proof}
Define $\Gamma g := \E_Q[g(X)\,(\phi(X) - \mu_Q)]$ for $g \in L^2(Q)$. Its adjoint is $\Gamma^* v = \inner{v}{\phi(\cdot) - \mu_Q}_{\Hs}$, so $\Gamma \Gamma^* = \Sigma$. Since $\E_Q[h] = 0$, $\delta = \E_Q[h(X)\, \phi(X)] = \Gamma h$. The polar decomposition gives $\Gamma = \Sigma^{1/2} W$ with $\|W\| \le 1$, so $\delta = \Sigma^{1/2} u$ with $u = W h$ and $\normH{u} \le \|h\|_{L^2(Q)}$. The identity for $\|h\|_{L^2(Q)}^2$ follows from $dQ = \tfrac12 d(P_a + P_b)$, and $|h| \le 2$ gives $\|h\|_{L^2(Q)}^2 \le 2\,\E_Q|h| = 2 \int |dP_a - dP_b| = 4\, \mathrm{TV}(P_a, P_b)$.
\end{proof}

For a characteristic kernel, the radius is exact. Then $\Gamma$ is injective on functions with mean zero, since $\Gamma g = 0$ makes the signed measure $g\, dQ$ have zero embedding, so $W$ preserves the norm of $h$ and $\normH{\Sigma^{-1/2}\delta} = \|h\|_{L^2(Q)}$. No smaller radius works for $r = 1/2$.

\paragraph{When the worst case is attained.}
\Cref{thm:pokemon-quant} takes the worst case over every vector in the ellipsoid $B_r(R)$, and not every such vector is a difference between two distributions. The lemma below shows that, for the best $m$ criteria and $r = 1/2$, a pair of distributions with the same pooled covariance attains the worst case when the eigenfunctions are uniformly bounded. Let $\psi_j := \inner{\phi(\cdot) - \mu_Q}{e_j}_{\Hs} / \lambda_j^{1/2}$ for $\lambda_j > 0$. These functions have mean zero and are orthonormal in $L^2(Q)$, and $\Gamma \psi_j = \lambda_j^{1/2} e_j$.

\begin{lemma}[Pairs that attain the worst case]
\label{lem:attained}
Suppose $|\psi_j| \le M$ for every $j$. For every $m$ with $\lambda_{m+1} > 0$, the distributions $P_a := (1 + \psi_{m+1}/M)\, Q$ and $P_b := (1 - \psi_{m+1}/M)\, Q$ have average $Q$, and hence pooled covariance $\Sigma$. They satisfy the criteria $e_1, \ldots, e_m$, and their residual unfairness $(2/M)\, \lambda_{m+1}^{1/2}$ equals the worst case of \cref{thm:pokemon-quant} over $B_{1/2}(R)$ with $R = 2/M$.
\end{lemma}

\begin{proof}
Both densities are nonnegative and integrate to one, since $|\psi_{m+1}/M| \le 1$ and $\E_Q[\psi_{m+1}] = 0$. Their difference is $\delta = (2/M)\, \Gamma \psi_{m+1} = (2/M)\, \lambda_{m+1}^{1/2} e_{m+1}$, which is orthogonal to $e_1, \ldots, e_m$ and lies in $B_{1/2}(2/M)$.
\end{proof}

Since $\psi_{m+1}$ has unit norm in $L^2(Q)$, $M \ge 1$, so the radius $2/M$ is at most the radius 2 of \cref{prop:source-half}. A bounded kernel always gives $|\psi_j| \le 2\kappa / \lambda_j^{1/2}$ (\cref{app:infinite}), but this bound grows with $j$. Without a uniform bound, the radius $2/\|\psi_{m+1}\|_\infty$ that a pair attains can shrink with $m$.

\subsection{Inexact Criteria}
\label{app:inexact}

When criteria hold only up to a tolerance, $|\inner{v_i}{\delta}_{\Hs}| \le \varepsilon$, they bound only the part $P_V \delta$ of the difference that the checklist sees. Pythagoras gives $\normH{P_{V^\perp}\delta}^2 = \MMD^2(P_a, P_b) - \normH{P_V\delta}^2$, which is~\eqref{eq:residual-decomposition}. For orthonormal criteria, $\normH{P_V \delta}^2 = \sum_{i=1}^m \inner{v_i}{\delta}_{\Hs}^2 \le m \varepsilon^2$. For criteria that are not orthonormal, let $K_{ij} := \inner{v_i}{v_j}_{\Hs}$ be their Gram matrix and $z_i := \inner{v_i}{\delta}_{\Hs}$ their readings. Then $\normH{P_V \delta}^2 = z^\top K^\dagger z$, where $K^\dagger$ is the Moore--Penrose pseudoinverse, so $\normH{P_V \delta} \le \varepsilon \sqrt{m / \lambda_{\min}^+(K)}$, where $\lambda_{\min}^+(K)$ is the smallest positive eigenvalue of $K$. Whenever the residual is nonzero, the unit vector $P_{V^\perp}\delta / \normH{P_{V^\perp}\delta}$ is a criterion orthogonal to the checklist, and it detects a difference of $\normH{P_{V^\perp}\delta}$ between the groups.

\subsection{Criteria in Different Distributions}
\label{app:dictionary}

\Cref{sec:pokemon-other} applies the Pok\'emon theorem to criteria that compare distributions other than that of $X$. Each criterion compares the expectation of a test function between the groups under one of three distributions. By the reproducing property, each comparison is an inner product with the difference between the groups' mean embeddings of that distribution.

\paragraph{The distributions of $X$ and of $X$ given $Y$.}
Demographic parity compares $\E[c(X) \mid G = g]$ between the groups, with $c(x) = \Ind[f(x) = 1]$. Equal error rates compare $\E[c(X) \mid Y = y, G = g]$ with the same $c$: the true-positive rate for $y = 1$ and the false-positive rate for $y = 0$. For $c = \inner{v}{\phi(\cdot)}_{\Hs}$, the comparisons are $\inner{v}{\delta}_{\Hs}$ and $\inner{v}{\delta_y}_{\Hs}$, with $\delta_y$ from~\eqref{eq:cmd-class}.

\paragraph{The joint distribution of $(X, Y)$.}
Calibration involves the outcome itself, so it needs a kernel on pairs $(x, y)$. We use $k((x, y), (x', y')) := k(x, x')\, \Ind[y = y']$. Its RKHS consists of the functions $f(x, y) = f_y(x)$ with $f_0, f_1 \in \Hs$, and its feature map sends $(x, y)$ to the pair whose entry $y$ is $\phi(x)$ and whose other entry is zero. The mean embedding of group $g$'s joint distribution is therefore the pair $\nu_g$ with entries $\Prb[Y = y \mid G = g]\, \mu_{y, g}$ for $y = 0, 1$, the embeddings of the distribution of $X$ jointly with $Y = y$. If $k$ is injective on finite signed measures, as the Gaussian, Laplace and Mat\'ern kernels are, each part determines that distribution, so the kernel on pairs is characteristic.

Calibration of a score $S = s(X)$ within group $g$, $\E[Y \mid S, G = g] = S$, holds if and only if $\E[c_h(X, Y) \mid G = g] = 0$ for every bounded function $h$, where $c_h(x, y) := (y - s(x))\, h(s(x))$, since $\E[c_h(X, Y) \mid G = g] = \E[(\E[Y \mid S, G = g] - S)\, h(S) \mid G = g]$. It suffices to check $h$ in any class that is dense in $L^2$ of the distribution of $S$ in group $g$. When $x \mapsto (y - s(x))\, h(s(x))$ lies in $\Hs$ for $y = 0, 1$, the test function $c_h$ lies in the RKHS on pairs, and $\E[c_h(X, Y) \mid G = g] = \inner{c_h}{\nu_g}$. Calibration within both groups then implies $\inner{c_h}{\nu_a - \nu_b} = 0$ for every such $h$. Calibration therefore implies a family of criteria in the joint distribution, and the Pok\'emon theorem applies to them with $\nu_a - \nu_b$ in place of $\delta$. Its member $h \equiv 1$ says that each group's average score equals its base rate. \Cref{app:infinite} shows that this family, although infinite, certifies nothing beyond the score.

\subsection{Infinitely Many Criteria}
\label{app:infinite}

Satisfying every criterion in a family $C \subset \Hs$ makes $\delta$ orthogonal to the closed span of $C$. If this span is all of $\Hs$, then $\delta = 0$, and for a characteristic kernel $P_a = P_b$: a dense family certifies. If instead some eigenvector $e_j$, with eigenvalue $\lambda_j > 0$, of the covariance operator of a distribution $Q$ is orthogonal to every criterion in $C$, the family certifies nothing, however many criteria it contains. With $\psi_j$ as in \cref{app:source}, $|\psi_j(x)| \le \normH{\phi(x) - \mu_Q} / \lambda_j^{1/2} \le 2\kappa / \lambda_j^{1/2}$, so for $0 < t \le \lambda_j^{1/2}/(2\kappa)$ the distributions $P_a := (1 + t\, \psi_j)\, Q$ and $P_b := (1 - t\, \psi_j)\, Q$ exist. They differ, with $\delta = 2t\lambda_j^{1/2} e_j$, and satisfy every criterion in $C$.

Density is sufficient for certification but not necessary. For a kernel that is injective on finite signed measures, the family $\{c \in \Hs : c(x_0) = 0\}$ is not dense, since $\phi(x_0)$ is orthogonal to it, yet it certifies. Satisfying it forces $\delta = t\,\phi(x_0)$ for some $t \in \R$, the mean embedding of $t$ times the point mass at $x_0$. Injectivity then makes $P_a - P_b$ equal to $t$ times this point mass, and since both groups have total mass one, $t = 0$.

\paragraph{Calibration.}
Calibration of a score $S = s(X)$ within each group is the family $\E[(Y - S)\, h(S) \mid G = g] = 0$ over bounded functions $h$ (\cref{app:dictionary}). It contains infinitely many criteria, but each depends on the data only through the pair $(S, Y)$. The proposition below shows that two groups can agree on every such criterion yet have different features.

\begin{proposition}[Criteria that see only the score and the outcome]
\label{prop:score-audit}
Let $s$ be measurable, and let $\pi$ be a joint distribution of $(X, Y)$ under which $X$ carries information beyond $(S, Y)$: some bounded test function $c$ has $\Prb_\pi[c(X) \ne \E_\pi[c(X) \mid S, Y]] > 0$. Then there is a joint distribution $\pi'$ that gives every bounded function of $(S, Y)$ the same expectation as $\pi$, but a different distribution of $X$.
\end{proposition}

\begin{proof}
Let $\bar c(S, Y) := \E_\pi[c(X) \mid S, Y]$ and $\theta(x, y) := c(x) - \bar c(s(x), y)$. Then $\theta$ is bounded, nonzero on a set of positive probability, and $\E_\pi[\theta \mid S, Y] = 0$. For $0 < t \le 1/\|\theta\|_\infty$, let $d\pi' := (1 + t \theta)\, d\pi$, a probability distribution. For every bounded function $\xi$ of $(S, Y)$, $\E_{\pi'}[\xi(S, Y)] = \E_\pi[\xi(S, Y)] + t\, \E_\pi[\xi(S, Y)\, \E_\pi[\theta \mid S, Y]] = \E_\pi[\xi(S, Y)]$.

Under $\pi'$, the distribution of $X$ has density $1 + t\,(c(x) - Tc(x))$ with respect to that under $\pi$, where $Tc(x) := \E_\pi[\bar c(S, Y) \mid X = x]$. If $Tc = c$, then $\E_\pi[c(X)^2] = \E_\pi[c(X)\, \bar c(S, Y)] = \E_\pi[\bar c(S, Y)^2]$, so $\E_\pi[(c(X) - \bar c(S, Y))^2] = 0$, contrary to the assumption. Hence the two distributions of $X$ differ.
\end{proof}

Taking $\pi$ and $\pi'$ as the joint distributions of the two groups, both groups have the same base rate, the same distribution of scores and the same error rates at every threshold. Every calibration criterion takes the same value in both groups, so the score is calibrated in one group exactly when it is calibrated in the other. Yet their features differ, and with a characteristic kernel the Pok\'emon witness detects the difference.

\section{Proofs for \texorpdfstring{\cref{sec:recover}}{Section 4}}
\label{app:recover}

This appendix proves the results of \cref{sec:recover}. Throughout, $S \in [0, 1]$ is a score, the kernel on scores is bounded and characteristic, and $\mu_g$, $\mu_{y, g}$, $\delta$ and $\delta_y$ are the mean embeddings of the score's distributions, as in \cref{sec:recover-notions}. For each group $g$, $\bar g$ denotes the other group. Identity~\eqref{eq:coupling} also holds with $a$ and $b$ exchanged,
\begin{equation}
  \delta \;=\; p_b\, \delta_1 + (1 - p_b)\, \delta_0 + (p_a - p_b)\,(\mu_{1, a} - \mu_{0, a}),
  \label{eq:coupling-a}
\end{equation}
as substituting $\mu_{y, b} = \mu_{y, a} - \delta_y$ into $\delta = \mu_a - \mu_b$ shows.

\subsection{Dense Checklists Certify the Three Notions}
\label{app:recover-dense}

\Cref{sec:recover-notions} states that a dense checklist for each distribution of the score certifies the corresponding notion. Write $P_a$ and $P_b$ for the groups' distributions of $S$, and $Q := \tfrac12(P_a + P_b)$.

\emph{Independence (C1).} The difference $P_a - P_b$ has a density with respect to $Q$ that is bounded by 2 in absolute value and integrates to zero. The constant function and the indicators $\Ind[s > t]$ span the step functions, which are dense in $L^2(Q)$. Equal selection rates at every threshold make the density orthogonal in $L^2(Q)$ to each of these functions, hence to all of $L^2(Q)$, so it is zero. The two distributions are equal, and $\delta = 0$.

\emph{Separation (C2).} The same argument within each outcome class shows that equal error rates at every threshold give $\delta_0 = \delta_1 = 0$.

\emph{Sufficiency (C3).} For every bounded $h$, $\E[(Y - f(S))\, h(S) \mid G = g] = \E[(\E[Y \mid S, G = g] - f(S))\, h(S) \mid G = g]$. This vanishes for every $h$ if and only if $\E[Y \mid S, G = g] = f(S)$ almost surely.

\subsection{Proofs of \texorpdfstring{\cref{thm:recover,thm:kmr}}{Theorems 4.1 and 4.2}}
\label{app:recover-proofs}

\paragraph{Averages of the score.}
Parts~(ii) and~(iii) of \cref{thm:recover} and \cref{thm:kmr} use only averages of the score. Write $\bar S_{y, g} := \E[S \mid Y = y, G = g]$ for the average score in outcome class $y$ of group $g$. Its distances from a perfect score, $\FNR_g := 1 - \bar S_{1, g}$ and $\FPR_g := \bar S_{0, g}$, are the score's error rates; for a binary decision they are the usual false-negative and false-positive rates. Let $d_y := \bar S_{y, a} - \bar S_{y, b}$ be the difference between the groups' average scores in class $y$, and $\zeta_g := p_g - \E[S \mid G = g]$ the gap between group $g$'s base rate and its average score. Calibration within group $g$ gives $\zeta_g = 0$.

The law of total expectation gives $p_g - \zeta_g = \E[S \mid G = g] = p_g\, \bar S_{1, g} + (1 - p_g)\, \bar S_{0, g}$ for each group. For $g = a$, substitute $\bar S_{y, b} = \bar S_{y, a} - d_y$ into the equation for $b$ and subtract it from the equation for $a$. This gives $(p_a - p_b)(\bar S_{1, a} - \bar S_{0, a}) + p_b\, d_1 + (1 - p_b)\, d_0 = (p_a - p_b) - (\zeta_a - \zeta_b)$, and since $\bar S_{1, a} - \bar S_{0, a} = 1 - \FNR_a - \FPR_a$,
\begin{equation}
  (p_a - p_b)\,(\FNR_g + \FPR_g) \;=\; (\zeta_a - \zeta_b) + p_{\bar g}\, d_1 + (1 - p_{\bar g})\, d_0 .
  \label{eq:app-kmr-identity}
\end{equation}
The case $g = b$ is symmetric.

\paragraph{Proof of \cref{thm:recover}.}
(i) Independence and separation give $\delta = \delta_0 = \delta_1 = 0$, so~\eqref{eq:coupling} leaves $(p_a - p_b)\, \omega = 0$, and $\mu_{1, b} = \mu_{0, b}$. Since the kernel is characteristic, $S$ has the same distribution in the two outcome classes of group $b$. By separation, the same holds in group $a$, so the four distributions of $S$ given $Y$ and $G$ are equal. The distribution of $S$ given $Y = y$ mixes those of the two groups, so it is the same for $y = 0$ and $y = 1$, and $S$ is independent of $Y$.

(ii) Let $T := f(S) \in [0, 1]$. Sufficiency gives $\E[Y \mid T, G = g] = \E[\E[Y \mid S, G = g] \mid T, G = g] = T$, so $T$ is calibrated within each group, and $\zeta_a = \zeta_b = 0$ for $T$. Separation of $S$ implies separation of $T$, a function of $S$, so $d_0 = d_1 = 0$ for $T$. Then~\eqref{eq:app-kmr-identity} for $T$ gives $\FNR_g + \FPR_g = 0$ in each group. Both terms are nonnegative, so $\E[1 - T \mid Y = 1, G = g] = 0$ and $\E[T \mid Y = 0, G = g] = 0$, and since $T \in [0, 1]$, $T = Y$ almost surely. A calibrated score has $f(s) = s$, so $S = Y$.

(iii) Sufficiency gives $\E[f(S) \mid G = g] = \E[Y \mid G = g] = p_g$, and independence makes these averages equal, so $p_a = p_b$. \hfill$\square$

\paragraph{Proof of \cref{thm:kmr}, and the version of \citet{stewart2024impossibility}.}
KMR assume calibration within each group, so $\zeta_a = \zeta_b = 0$, and equal average scores within each outcome class, so $d_0 = d_1 = 0$. Then~\eqref{eq:app-kmr-identity} gives $\FNR_g + \FPR_g = 0$, and $S = Y$ almost surely as in~(ii). The proof uses only $\zeta_a = \zeta_b$, so it also proves the theorem of \citet{stewart2024impossibility}, who replace calibration by equal gaps between average score and base rate in the two groups. \hfill$\square$

\paragraph{Relation to the original statements.}
\citet{kleinberg2017inherent} state their theorem for binned, possibly randomized scores, and conclude that the instance admits perfect prediction; their proof shows that the score itself is perfect. In \citet{barocas2023fairness}, (i) and (iii) of \cref{thm:recover} are Propositions 3 and 2. Their Proposition 4 states that separation and sufficiency cannot both hold when every event of $(G, S, Y)$ has positive probability, an assumption that excludes perfect scores; (ii) shows that the perfect score is the only exception.

\paragraph{Equal error rates and equal predictive values.}
For a binary decision $\hat Y$ that is positive with positive probability, separation means common true- and false-positive rates $\TPR$ and $\FPR$ in both groups. Equal positive predictive value means $\E[(Y - \PPV)\, \hat Y \mid G = g] = 0$ in both groups, a criterion linear in the distribution. Since $\E[Y \hat Y \mid G = g] = p_g \TPR$ and $\E[\hat Y \mid G = g] = p_g \TPR + (1 - p_g) \FPR$, it reads $\frac{p_g}{1 - p_g} \TPR\, (1 - \PPV) = \FPR\, \PPV$. The right side is the same in both groups, so when $p_a \ne p_b$, $\TPR\, (1 - \PPV) = 0$: either $\TPR = 0$, or $\PPV = 1$ and $\FPR = 0$. This is the result of \citet{chouldechova2017fair}: unless a decision has no true positives or no false positives, equal error rates and equal predictive values are incompatible when base rates differ.

\subsection{Approximate Impossibilities}
\label{app:approx-impossibilities}

When the criteria hold only up to tolerances, the conclusions of \cref{thm:recover} hold approximately, with an error proportional to the tolerances. The gaps $\zeta_g$ and $d_y$ measure how far calibration and equal class averages are from holding, and $\normH{\delta}$ and $\normH{\delta_y}$ how far independence and separation are.

\begin{theorem}[Approximate impossibilities]
\label{thm:approx-impossibilities}
Let $S \in [0, 1]$, use any bounded kernel on scores for~(i), and suppose $p_a \ne p_b$. For each group $g$:
\begin{enumerate}[label=(\roman*), leftmargin=2em, itemsep=1pt, topsep=2pt]
  \item Independence and separation: $|p_a - p_b| \cdot \normH{\mu_{1, g} - \mu_{0, g}} \le \normH{\delta} + p_{\bar g}\, \normH{\delta_1} + (1 - p_{\bar g})\, \normH{\delta_0}$.
  \item Separation and calibration: $|p_a - p_b| \cdot \E\bigl[|S - Y| \bigm| G = g\bigr] \le |\zeta_a - \zeta_b| + p_{\bar g}\, |d_1| + (1 - p_{\bar g})\, |d_0|$.
  \item Independence and calibration: $|p_a - p_b| \le |\E[S \mid G = a] - \E[S \mid G = b]| + |\zeta_a - \zeta_b|$.
\end{enumerate}
\end{theorem}

\begin{proof}
(i) Take norms in~\eqref{eq:coupling} for $g = b$ and in~\eqref{eq:coupling-a} for $g = a$. (ii) Take absolute values in~\eqref{eq:app-kmr-identity}, and use $\E[|S - Y| \mid G = g] = p_g \FNR_g + (1 - p_g) \FPR_g \le \FNR_g + \FPR_g$. (iii) The definition of $\zeta_g$ gives $p_a - p_b = (\E[S \mid G = a] - \E[S \mid G = b]) + (\zeta_a - \zeta_b)$.
\end{proof}

With all tolerances zero, these give \cref{thm:recover}(i), and (ii) and (iii) for a calibrated score. If each group's average score is within $\tau$ of its base rate, then $|\zeta_a - \zeta_b| \le 2\tau$. By Markov's inequality, (ii) also bounds the tail, $\Prb[|S - Y| > t \mid G = g] \le (|\zeta_a - \zeta_b| + p_{\bar g} |d_1| + (1 - p_{\bar g}) |d_0|) / (|p_a - p_b|\, t)$. Approximate forms of KMR appear in \citet[Theorem 1.2]{kleinberg2017inherent} and \citet[Theorem 9]{pleiss2017fairness}, with multiplicative tolerances; (ii) is linear in the tolerances. Part~(iii) appears in the proof of Theorem 8 of \citet{zhaogordon2022inherent} for a binary classifier, and part~(i) sharpens \cref{thm:approx-fair-feature} for a score.

\subsection{Impossibilities with the Best \texorpdfstring{$m$}{m} Criteria}
\label{app:recover-best-m}

A score that satisfies only the best $m$ criteria of a notion satisfies that notion up to a tolerance, and \cref{thm:pokemon-quant} bounds the tolerance. To make the bound explicit, we need a kernel on scores whose covariance operators have a bounded trace and which measures average scores. The kernel $\min(s, s')$ has both properties.

\paragraph{The kernel $\min(s, s')$.}
On $[0, 1]$, $\min(s, s') = \int_0^1 \Ind[t < s]\, \Ind[t < s']\, dt$, so this kernel has the feature map $s \mapsto \Ind_{[0, s)}$ into $L^2[0, 1]$. Its RKHS consists of the functions $f$ with $f(0) = 0$ and $f' \in L^2[0, 1]$, with $\normH{f}^2 = \int_0^1 f'(t)^2\, dt$, and $k(s, s) \le 1$. The mean embedding of a distribution of $S$ is $\mu(t) = \E[\min(S, t)] = \int_0^t \Prb[S > \tau]\, d\tau$. It determines $\Prb[S > \tau]$ for every $\tau$, so the kernel is characteristic. The covariance operator of any distribution of $S$ has trace $\int_0^1 \Prb[S > t]\,(1 - \Prb[S > t])\, dt \le 1/4$. Finally, $\iota(s) := s = k(s, 1)$ has $\normH{\iota} = 1$ and $\inner{\iota}{\mu}_{\Hs} = \E[S]$, so MMDs bound differences of average scores: $|d_y| \le \normH{\delta_y}$ and $|\E[S \mid G = a] - \E[S \mid G = b]| \le \normH{\delta}$.

\paragraph{The best $m$ criteria.}
The best $m$ criteria of independence (C1) are the $m$ leading eigendirections of the covariance operator of the score's distribution, pooled over both groups. Those of separation (C2) are, for each $y$, the $m$ leading eigendirections of the covariance operator of its distribution given $Y = y$, pooled over both groups.

\begin{corollary}[Impossibilities with the best $m$ criteria]
\label{thm:recover-approx}
Use the kernel $\min(s, s')$ on scores, let $\eta_m := 1/\sqrt{m + 1}$, and suppose $p_a \ne p_b$. For each group $g$:
\begin{enumerate}[label=(\roman*), leftmargin=2em, itemsep=1pt, topsep=2pt]
  \item If $S$ satisfies the best $m$ criteria of independence (C1) and of separation (C2), its predictive signal in group $g$ satisfies $|p_a - p_b| \cdot \normH{\mu_{1, g} - \mu_{0, g}} \le 2\eta_m$.
  \item If $S$ satisfies the best $m$ criteria of separation (C2), then $|p_a - p_b| \cdot \E\bigl[|S - Y| \bigm| G = g\bigr] \le |\zeta_a - \zeta_b| + \eta_m$. For a calibrated score, $\zeta_a = \zeta_b = 0$.
  \item If $S$ satisfies the best $m$ criteria of independence (C1), then $|p_a - p_b| \le |\zeta_a - \zeta_b| + \eta_m$.
\end{enumerate}
\end{corollary}

\begin{proof}
By \cref{thm:approx-impossibilities} and the bounds on $|d_y|$ and on the difference of average scores above, it suffices to show $\normH{\delta} \le \eta_m$ and $\normH{\delta_y} \le \eta_m$. Satisfying the best $m$ criteria of independence places $\delta$ orthogonal to the leading $m$ eigenvectors of the covariance operator $\Sigma$ of $Q = \tfrac12(P_a + P_b)$, the pooled distribution of $S$. By \cref{prop:source-half}, $\delta = \Sigma^{1/2} u$ with $\normH{u} \le 2$, so the upper bound of \cref{thm:pokemon-quant} gives $\normH{\delta} \le 2 \lambda_{m+1}(\Sigma)^{1/2}$. Since $(m + 1)\,\lambda_{m+1} \le \lambda_1 + \cdots + \lambda_{m+1} \le \operatorname{tr}\Sigma \le 1/4$, this is at most $2 \cdot \tfrac12 (m + 1)^{-1/2} = \eta_m$. The same argument, applied to the distributions of $S$ given $Y = y$, gives $\normH{\delta_y} \le \eta_m$.
\end{proof}

As $m \to \infty$, (i) gives \cref{thm:recover}(i), and (ii) and (iii) give \cref{thm:recover}(ii) and (iii) for a calibrated score. The bounds shrink at the rate $m^{-1/2}$, or $m^{-\alpha/2}$ if \textbf{(A2)} holds for the three pooled covariance operators. The best $m$ criteria depend on the score's own distributions, so the corollary describes the best possible checklist. It is a result about rates: it shows how quickly the conclusions of \cref{thm:recover} return as a score satisfies more of the best criteria. Its bounds hold for every score and do not adapt to the data, so they are not a certificate for a particular score.

\section{Proofs for \texorpdfstring{\cref{sec:fair-feature,sec:approx}}{Sections 5 and 6}}
\label{app:proofs-features}

This appendix proves the results of \cref{sec:fair-feature,sec:approx}. Throughout, the kernel on the feature space $\Zs$ is bounded and characteristic, and $\Hs$ denotes its RKHS. As in \cref{sec:approx}, $\mu_g$, $\mu_{y, g}$, $\delta$ and $\delta_y$ are the mean embeddings of the distributions of the features $Z = \Phi(X)$, not of $X$, and $\omega_g := \mu_{1, g} - \mu_{0, g}$. Identities~\eqref{eq:coupling} and~\eqref{eq:coupling-a} hold for $Z$ as for $S$, since their derivation uses only the law of total expectation.

\subsection{Proof of \texorpdfstring{\cref{thm:fair-feature}}{Theorem 5.1}}
\label{app:fair-feature-proof}

The proof is that of \cref{thm:recover}(i), with the features $Z$ in place of the score $S$. Independence (C1) and separation (C2) give $\delta = \delta_0 = \delta_1 = 0$, so~\eqref{eq:coupling} leaves $(p_a - p_b)\, \omega_b = 0$. Since $p_a \ne p_b$, $\mu_{1, b} = \mu_{0, b}$, and since the kernel is characteristic, $Z$ has the same distribution in the two outcome classes of group $b$. By separation, the same holds in group $a$, so the four distributions of $Z$ given $Y$ and $G$ are equal, and $Z$ is independent of $Y$. Hence every function $f$ of the features with $\E|f(Z)| < \infty$ has $\E[f(Z) \mid Y = 1] = \E[f(Z) \mid Y = 0]$, within each group and overall: no model built on $Z$ separates the two outcome classes, even on average. \hfill$\square$

\subsection{Proofs of \texorpdfstring{\cref{thm:approx-fair-feature} and \cref{prop:frontier}}{Theorem 6.1 and Proposition 6.2}}
\label{app:approx-identity}

\paragraph{Proof of \cref{thm:approx-fair-feature}.}
By~\eqref{eq:coupling} and~\eqref{eq:coupling-a},
\[
  (p_a - p_b)\, \omega_b \;=\; \delta - p_a\, \delta_1 - (1 - p_a)\, \delta_0,
  \qquad
  (p_a - p_b)\, \omega_a \;=\; \delta - p_b\, \delta_1 - (1 - p_b)\, \delta_0.
\]
Taking norms, with $\normH{\delta} = \varepsilon$ and $\normH{\delta_y} \le \rho$, gives $|p_a - p_b|\, \normH{\omega_g} \le \varepsilon + p_{\bar g}\, \rho + (1 - p_{\bar g})\, \rho = \varepsilon + \rho$ for each group $g$, where $\bar g$ is the other group. This is~\eqref{eq:approx-pareto}. \hfill$\square$

\paragraph{Proof of \cref{prop:frontier}.}
\emph{Separated features lie on the bound.} Let $Z$ satisfy separation, so $\mu_{y, a} = \mu_{y, b}$ for each $y$, and both groups have the same signal $\omega := \omega_a = \omega_b$. Then~\eqref{eq:coupling} gives $\delta = (p_a - p_b)\, \omega$, so $\rho = 0$ and $\varepsilon = |p_a - p_b|\, \normH{\omega}$.

\emph{Mixing with a constant shrinks every difference.} Fix $z_0 \in \Zs$ and $\lambda \in [0, 1]$, and let $Z'$ equal $Z$ with probability $\lambda$ and $z_0$ otherwise, using a coin independent of $(X, Y, G)$. Each distribution of $Z'$, overall or given $Y$ and $G$, mixes the corresponding distribution of $Z$ with the point mass at $z_0$. Its mean embedding is therefore $\lambda \mu + (1 - \lambda)\, \phi(z_0)$, where $\mu$ is that of $Z$ and $\phi$ is the feature map of the kernel on $\Zs$. Every difference between two of these mean embeddings is multiplied by $\lambda$: $Z'$ satisfies separation, its independence violation is $\lambda\varepsilon$, and its signal is $\lambda\,\normH{\omega}$ in both groups.

\emph{Every point below $\|\omega\|_{\mathrm{sep}}$ is reached.} Let $0 \le \beta < |p_a - p_b|\, \|\omega\|_{\mathrm{sep}}$. By the definition of the supremum, there are separated features $Z$ with $\normH{\omega} > \beta / |p_a - p_b|$. Mixing them with $\lambda := \beta / (|p_a - p_b|\, \normH{\omega}) < 1$ gives features with $\rho = 0$, $\varepsilon = \beta$, and signal $\beta / |p_a - p_b|$ in both groups. By \cref{thm:approx-fair-feature}, no features with $\varepsilon + \rho \le \beta$ have more signal in either group. \hfill$\square$

\subsection{Error of Classifiers with Equal Error Rates}
\label{app:pareto-bound}

For a binary decision that satisfies separation, \eqref{eq:coupling} ties the gap in selection rates to the decision's signal, and this gives a lower bound on its error. \citet{menon2018cost} characterize the Bayes-optimal trade-off between accuracy and fairness for given constraints; the following bound holds for every distribution.

\begin{corollary}[Error under equal error rates]
\label{cor:pareto-bound}
Suppose $p_a \ne p_b$, and let $p := \Prb[Y = 1]$. Let a decision $\hat Y \in \{0, 1\}$ satisfy separation, with true- and false-positive rates $\TPR$ and $\FPR$ common to both groups, and let $\Delta_{\mathrm{sel}} := |\Prb[\hat Y = 1 \mid G = a] - \Prb[\hat Y = 1 \mid G = b]|$ be its gap in selection rates. Then
\begin{equation}
  \Prb[\hat Y \ne Y] \;\ge\; \min(p, 1 - p)\, \Bigl(1 - \frac{\Delta_{\mathrm{sel}}}{|p_a - p_b|}\Bigr),
  \label{eq:pareto-binary}
\end{equation}
with equality when $p \le 1/2$ and $\FPR = 0$, or $p \ge 1/2$ and $\TPR = 1$.
\end{corollary}

\begin{proof}
By the law of total expectation, $\Prb[\hat Y = 1 \mid G = g] = \FPR + p_g\,(\TPR - \FPR)$, so $\Delta_{\mathrm{sel}} = |p_a - p_b|\, |\TPR - \FPR|$. The error is
\[
  p\,(1 - \TPR) + (1 - p)\,\FPR \;\ge\; \min(p, 1 - p)\,(1 - \TPR + \FPR) \;\ge\; \min(p, 1 - p)\,(1 - |\TPR - \FPR|),
\]
which is~\eqref{eq:pareto-binary}. Both inequalities are equalities under the stated conditions.
\end{proof}

\section{Experimental Details}
\label{app:experiments}

This appendix describes the two experiments of \cref{sec:experiments}: E1, the residual after the best $m$ criteria (\cref{fig:experiments}A), and E2, fairness and predictive signal (\cref{fig:experiments}B and C). Full protocols of both experiments will be released with the code.

\subsection{Data, Splits and Intervals}
\label{app:experiments-datasets}
\label{app:experiments-stats}
\label{app:experiments-kernels}

\paragraph{Datasets.}
Adult Income (UCI) predicts income above \$50K, with men and women as the groups ($|\Delta p| = .199$). COMPAS (ProPublica) predicts two-year recidivism, with Black and white defendants as the groups, after the standard filters ($5{,}278$ rows, $|\Delta p| = .132$). ACS PUMS uses the ACSIncome task on California 2018 data from \texttt{folktables} \citep{ding2021retiring}, with white and non-white respondents as the groups ($|\Delta p| = .086$). The group is a function of the race column, so we remove that column from the features. The two experiments use ACS PUMS differently. For each seed, E1 draws a subsample of at most $20{,}000$ rows, stratified by outcome and group (all rows on COMPAS), and one-hot encodes the ACS features. E2 uses numeric features on all $195{,}665$ ACS rows.

\paragraph{Splits and intervals.}
Both experiments use the same ten seeds. In E2, each seed splits a dataset into training, validation and test rows (70/15/15), stratified by outcome and group. Intervals are 95\% two-level bootstrap intervals of the mean over seeds ($B = 1000$; seeds, then rows stratified by outcome and group). Seeds re-split the same rows, so these intervals are conditional on the dataset. On COMPAS, where every seed uses all rows, E1 instead uses half-widths from 200 resamples of the whole dataset.

\subsection{Reading \texorpdfstring{\cref{fig:experiments}}{Figure 1}}
\label{app:experiments-reading}

\paragraph{Two held-out statistics.}
E2 can average its test-row values over seeds in two ways. The 10-seed mixture applies the rule of a randomly chosen seed. The mean over seeds averages what each seed's fitted rule does, which noise biases upward. A single fitted rule lies between the two. All test-row marks in \cref{fig:experiments} are 10-seed mixtures, and we never compare a value of one statistic with a value of the other. The 10-seed mixture was adopted after seeing the figures; by the seed mean, the test at $\|\omega\|_{\mathrm{sep}}$ gives $p = .045$ on Adult and COMPAS.

\paragraph{Panel A.}
The blue line is the residual share $\mathrm{res}(m)$ of \cref{app:experiments-exp4} after the top $m$ eigendirections, averaged over seeds, with its 95\% interval. The red line is the worst case of \cref{thm:pokemon-quant} with $r = 1/2$ and $R = 2$ (\cref{prop:source-half}); on this axis it is $4\lambda_{m+1} / \MMD^2(P_a, P_b)$, with the eigenvalues estimated on separate rows. It is a worst case over all pairs of groups with this pooled covariance, not a prediction for this pair, so the panel does not test the theorem. The blue line stops where it falls into the gray noise level.

\paragraph{Panels B and C.}
Both axes use the rates of \cref{sec:approx} (selection-rate and error-rate gaps, and $\TPR_g - \FPR_g$), measured on test rows, and the red line is the bound of \cref{thm:approx-fair-feature}. The bound holds on any set of rows, so that no point lies above it follows by algebra and is not evidence for the theorem. In panel B, each point is one encoder at one setting of its fairness weight, and a logistic regression trained on its features, thresholded at $1/2$, makes the decisions. In panel C, the line is the frontier of the linear program on the rows it is fit to, and the circles apply the same rules to test rows; compare methods with the circles, not with the line. The black point is $\|\omega\|_{\mathrm{sep}}$ on the fit rows. Below it, the circles are rescalings of one rule (\cref{prop:frontier}) and give a single test. The gray band is where a rule exactly on the bound, refit on each seed, lands on test rows with probability 95\%; it stops at $\|\omega\|_{\mathrm{sep}}$, beyond which no rule lies on the bound. The triangles are exponentiated gradient at tolerances .005 and .02, the diamond is post-processing, and the gray points repeat panel B (\cref{app:experiments-exp3}).

\subsection{E1: Residual After the Best \texorpdfstring{$m$}{m} Criteria}
\label{app:experiments-exp4}

\paragraph{Protocol.}
We use a Gaussian kernel on standardized features with the median bandwidth ($0.5\times$ and $2\times$ in \cref{fig:app-e1-bandwidth}), and pool the groups with equal weights, $Q = \tfrac12 (P_a + P_b)$. Each seed draws three splits of its rows into three parts: $A$ (half the rows) estimates the eigendirections, and $B_1$ and $B_2$ (a quarter each) give two independent estimates $\hat\delta_1$ and $\hat\delta_2$ of the difference between the groups. For orthonormal eigendirections $q_1, \ldots, q_m$ estimated on $A$, the residual share is
\[
  \mathrm{res}(m) \;=\; 1 - \frac{\sum_{l \le m} \inner{\hat\delta_1}{q_l}_{\Hs}\inner{\hat\delta_2}{q_l}_{\Hs}}{\widehat{\MMD}{}^2_U},
\]
where $\widehat{\MMD}{}^2_U$ is the unbiased estimate on $B_1 \cup B_2$. Given $A$, the numerator and the denominator are unbiased, so $\mathrm{res}(m)$ can fall below zero when the true residual is below the noise. The noise level is the 90th percentile of the captured share, in absolute value, over 999 permutations of the group labels within $B_1$ and $B_2$. The witness is the difference between the groups estimated on $A$, used as a single criterion and evaluated on $B_1 \cup B_2$.

\paragraph{Results.}
The top $m$ eigendirections leave 10\% of $\MMD^2$ after 12 [12, 14] criteria on Adult, 2 [2, 3] on COMPAS and 30 [30, 31] on ACS PUMS. They leave 1\% after 29 [28, 32] criteria on Adult; on COMPAS and ACS PUMS, 1\% is not resolved above the noise level. The worst case first falls below 1 at $m = 11$ on Adult, 8 on COMPAS and 39 on ACS PUMS. The witness, a single criterion built for the pair, leaves .0056 [.0039, .0074] of $\MMD^2$ on Adult, .017 ($\pm .037$) on COMPAS and .042 [.024, .058] on ACS PUMS, which is .83 to .93 of its expected estimation noise at every setting. A checklist chosen without the difference between the groups thus needs 2 to 30 criteria to leave 10\%, while one criterion built for the pair leaves only noise. With numeric features on ACS PUMS, as in E2, the top $m$ eigendirections leave 10\% after 6 [6, 6] criteria and 1\% after 17 [14, 20].

\paragraph{Input features as criteria.}
\Cref{fig:app-e1-columns} uses criteria a practitioner might check: equal averages of the input features, added one at a time. These criteria are not directions in the Gaussian RKHS, so we measure what they leave without projecting them onto it. Let $h = dP_a/dQ - dP_b/dQ$ as in \cref{app:source}, and let $h'$ be the result of the smallest $\chi^2$ change to $h$ that equalizes the first $m$ feature averages. We report the share of $\MMD^2(P_a, P_b)$ left by the signed pair $(1 \pm h'/2)\,Q$. This pair is not two distributions: some rows carry negative weight, and each side holds less than full mass. With all features on the list, the share is .0059 on Adult, .023 on COMPAS and .030 on ACS PUMS with numeric features, where the pair keeps .44, .90 and .90 of its mass. Divided by the squared mass, every full list leaves about 3 to 7\% of $\MMD^2$, or 17 to 26\% of the MMD; the upper end, one-hot ACS PUMS, is biased upward by its fit of 550 columns. Only ACS PUMS with numeric features is clearly above its noise level: Adult is at its level, and COMPAS and one-hot ACS PUMS are below theirs. The order of the features matters. On Adult, the list leaves 10\% after 34 features in column order, but after 4 when the feature with the largest difference comes first, an order chosen using the difference between the groups. Projecting the features onto the RKHS by kernel ridge regression (their ridge image) does not work: it reads .71 to .85 on COMPAS even for a complete list.

\paragraph{Finite-dimensional kernels.}
\label{app:experiments-finite}
A linear kernel has a $d$-dimensional feature map, so $d$ criteria exhaust it (\cref{app:realizability}). On subsamples of about $5{,}000$ rows with five seeds, the residual after the $d$ linear criteria was zero to numerical precision on every dataset and seed, while $d$ eigendirections of a Gaussian kernel, chosen greedily to capture as much of the difference as possible, still left .001 to .112 of $\MMD^2$. After the groups were matched so that all $d$ feature averages agreed exactly, a Gaussian MMD test still separated them on all three datasets ($p = .001$ with 999 permutations), with an $\MMD^2$ 23 to 88 times the bias that the matching introduces. Satisfying every linear criterion certified that $d$ averages agree, not that the distributions do.

\paragraph{Stability of the witness.}
The estimated witness moves by at most $2\normH{\hat\delta - \delta} / \normH{\delta}$, so it is stable when the MMD is large compared with the sampling error. We refit it on resamples of growing size, using the continuous features of each dataset and five seeds. At the largest sizes ($16{,}000$, $4{,}000$ and $8{,}000$ rows), its cosine with a resampled witness was at least .998, .997 and .985 on Adult, COMPAS and ACS PUMS, against about .7 with randomly permuted group labels. The loss $1 - \cos$ fell roughly as $1/n$ on Adult and COMPAS (log--log slopes within .16 of $-1$), and more slowly on ACS PUMS, whose groups differ least (slopes $-.66$ to $-.81$); witnesses fit on disjoint samples gave more variable slopes. The witness scores correlate most with hours worked on Adult ($.83$), age on COMPAS ($-.74$) and age on ACS PUMS ($.89$); these are descriptive associations.

\begin{figure}[ht]
  \centering
  \includegraphics[width=\textwidth]{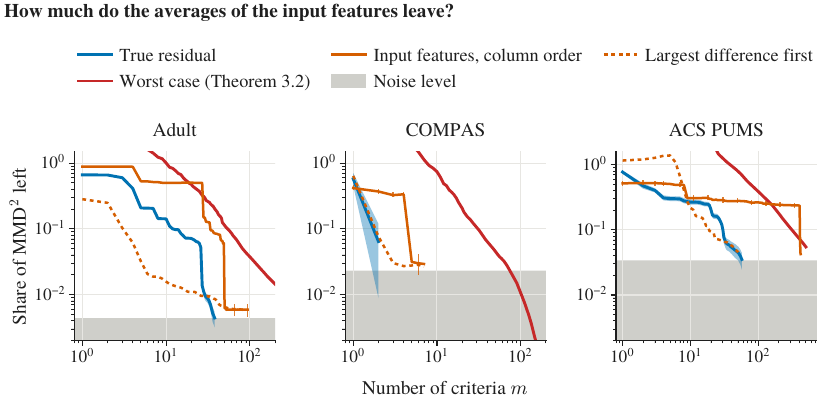}
  \caption{\textbf{Input features as criteria.} Orange shows the share of $\MMD^2$ left by the signed pair $(1 \pm h'/2)\,Q$ after the smallest $\chi^2$ change that equalizes the first $m$ feature averages; that pair is not two distributions and can hold well under full mass. Solid uses the data's column order, which is arbitrary; dotted takes the largest difference first, an order chosen using the difference between the groups. Blue repeats \cref{fig:experiments}A, and red is the worst case of \cref{thm:pokemon-quant}. ACS PUMS uses one-hot features. Curves stop at the larger of the noise level and the witness's residual. Vertical marks are 95\% intervals at grid points. Values can exceed 1: equalizing one feature's average can enlarge the difference elsewhere. The steep end on ACS PUMS comes from the order: the first place-of-birth columns arrive after 390 occupation columns.}
  \label{fig:app-e1-columns}
\end{figure}

\begin{figure}[ht]
  \centering
  \includegraphics[width=\textwidth]{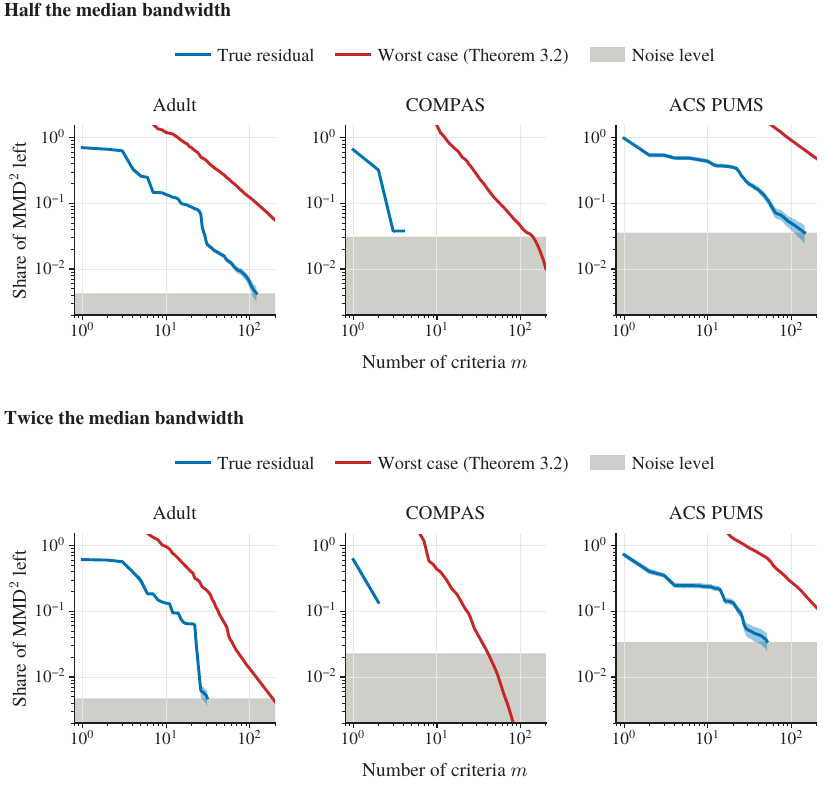}
  \caption{\textbf{Other bandwidths.} \Cref{fig:experiments}A at $0.5\times$ and $2\times$ the median bandwidth. COMPAS has no band: whole-dataset resamples exist only at $1\times$, and bands over seeds are too narrow there.}
  \label{fig:app-e1-bandwidth}
\end{figure}

\subsection{E2: Fairness and Predictive Signal}
\label{app:experiments-exp3}
\label{app:experiments-exp5}

\paragraph{Protocol.}
Every method fits on the same labeled rows, training and validation, and test rows are used only to evaluate. The base model is chosen by validation log-loss, unanimously over seeds: gradient-boosted trees \citep{chen2016xgboost} on Adult and ACS PUMS, and a neural network on COMPAS. The features never contain the group; the rules use it. The program of \cref{sec:approx} is fit on 5-fold out-of-fold scores. On test rows, a rule applies to each fold model's score and averages the decisions. Below $\|\omega\|_{\mathrm{sep}}$ the rule is the blend of \cref{prop:frontier}; above it, the program maximizes the signal of one group, then that of the other. The noise band is that of a rule exactly on the bound, refit on each seed, widened by $\sqrt{1 + n_T / n_F}$ for $n_T$ test and $n_F$ fit rows.

\paragraph{Results.}
The best signal any threshold rule on the score achieves in the two groups is .633 [.631, .635] and .730 [.727, .733] on Adult, .369 [.362, .374] and .306 [.298, .314] on COMPAS, and .644 [.643, .645] and .636 [.636, .637] on ACS PUMS. On all 30 pairs of dataset and seed, $\|\omega\|_{\mathrm{sep}}$ equals the smaller of the two. The ratio of the smaller to the larger, .868 [.863, .873], .829 [.803, .856] and .989 [.987, .991], is significantly below its value if both groups shared one ROC curve ($p = .001$, $.007$ and $.027$). It measures how differently predictable the groups are, not the cost of fairness. At $\|\omega\|_{\mathrm{sep}}$, held-out unfairness is 1.09, 1.87 and 1.005 times $|\Delta p|$ times the rule's held-out signal for the mixture, and 1.17, 2.60 and 1.10 times for the mean over seeds. Against the noise band, the one-sided $p$-values are .11, .11 and .90 for the mixture, and .045, .045 and .72 for the mean over seeds. On COMPAS the test has little power: it would pass a rule with about 2.4 times the unfairness on the bound, against about 1.15 times on Adult and 1.2 times on ACS PUMS, so the data cannot say whether COMPAS's rule lies on the bound.

\paragraph{Methods.}
Post-processing \citep{hardt2016equality} with a balanced objective solves the same program at $\|\omega\|_{\mathrm{sep}}$. Exponentiated gradient \citep{agarwal2018reductions} runs with demographic parity and with equalized odds over a grid of tolerances, with the group among its features. With a balanced objective and tolerances up to .02, it needs .95 to 1.08 times the program's unfairness to reach the same signal on test rows on Adult, .97 to 1.45 times on COMPAS and 1.04 to 1.89 times on ACS PUMS; the intervals on COMPAS are wide, from [.49, 1.35] to [.76, 1.71]. Looser tolerances, where active, reach 2.3, 8.3 and 1.8 times. With its default objective, accuracy, it ties a program that minimizes error under equalized odds, and under demographic parity on Adult its error is .014 to .019 higher.

The encoders are an adversarially trained encoder \citep{edwards2016censoring}, a variational encoder with an MMD penalty \citep{louizos2016variational}, and LFR \citep{zemel2013learning}, each trained on training rows over a grid of fairness weights. LFR is fit with exact gradients of the objective of the \texttt{aif360} implementation \citep{bellamy2019aif360}, which appends the group to the features. Run on a score fit to each encoder's features, the program keeps, at each method's best setting, .952, .931 and .978 (adversarial, variational, LFR) of the signal it keeps on the input features on Adult, 1.04, 1.00 and 1.06 on COMPAS, and 1.004, .995 and .977 on ACS PUMS. The COMPAS values are within noise of 1 and carry the upward bias of a maximum over groups. The program uses the group, so these ratios show that the features keep the inputs' signal, not that they remove the group; they are not evidence about \cref{thm:fair-feature}.

\subsection{Predictions That Failed}
\label{app:experiments-failed}

\begin{itemize}[leftmargin=1.5em, itemsep=1pt, topsep=2pt]
  \item \emph{A list of the input features' averages leaves far more than the top $m$ eigendirections at every $m$.} The planned test used the ridge image, which reads high even for complete lists, so it could not fail. On the ridge-free measure, adopted afterwards, the list leaves less than the top eigendirection at $m = 1$ on COMPAS and ACS PUMS at every bandwidth, by .14 to .42 of $\MMD^2$. On Adult, the list leaves at least twice as much at 35 of 38 resolved values of $m$.
  \item \emph{The frontier is a ceiling for held-out methods.} Exponentiated gradient exceeds it on ACS PUMS, with intervals below zero at all four tight tolerances, and on Adult by point estimate. Only the line of \cref{thm:approx-fair-feature} is a bound.
  \item \emph{Standard methods need more unfairness at tight tolerances.} This was mostly a mismatch of objectives.
  \item \emph{The share of the largest signal kept by $\|\omega\|_{\mathrm{sep}}$ measures the cost of separation.} It measures how differently predictable the groups are. On COMPAS, an earlier run with five seeds gave .87; with ten seeds it is .83 [.80, .86].
\end{itemize}

\subsection{Code and Reproducibility}
\label{app:experiments-code}

The pipeline is implemented in Python~3.12 with dependencies pinned by \texttt{uv}. It ran on one workstation with four RTX 4090 GPUs and 64 CPU cores, using 64-bit eigendecompositions. Code, protocols and figure scripts will be released under the MIT license on publication.

\end{document}